\pdfoutput=1
\documentclass[11pt,a4paper]{amsart}
\usepackage{amsaddr}
\usepackage{amsfonts}
\usepackage{amsthm}
\usepackage{amscd}
\usepackage[utf8]{inputenc}
\usepackage[mathscr]{eucal}
\usepackage{indentfirst}
\usepackage{graphicx}
\numberwithin{equation}{section}
\usepackage{enumitem}
\usepackage{thm-restate}
\usepackage[margin=2.9cm]{geometry}
\usepackage{mdframed}
\usepackage{dsfont}
\usepackage{macros}

\newmdtheoremenv{btheorem}{Theorem}

\newcommand{\myLambda}{\vec{\Lambda}}
\newcommand{\myOmega}{\vec{\Omega}}
\newcommand{\mySigma}{\vec{\Sigma}}
\newcommand{\Identity}{\vec{I}}

\newcommand{\pa}[1]{\operatorname{pa}(#1)}
\newcommand{\ch}[1]{\operatorname{ch}(#1)}
\newcommand{\pav}{\pa v}
\newcommand{\pah}[2]{\operatorname{pa}_{#1}(#2)}
\newcommand{\spa}{\operatorname{spa}(v)}

\newcommand{\myLambdaSub}[2]{\vec{\Lambda}_{[#1,\;#2]}}
\newcommand{\myLambdaSubSup}[3]{\vec{\Lambda}_{[#1,\;#2]}^{#3}}

\newcommand{\myLambdaTSub}[2]{\tilde{\vec{\Lambda}}_{[#1,\;#2]}} 

\newcommand{\mySigmaSub}[2]{\vec{\Sigma}_{[#1,\;#2]}}
\newcommand{\mySigmaTSub}[2]{\tilde{\vec{\Sigma}}_{[#1,\;#2)]}}

\newcommand{\myOmegaSub}[2]{\vec{\Omega}_{[#1,\;#2]}}

\newcommand{\myOmegaSubAppend}[4]{\vec{\Omega}_{[#1,\;#2,\;#3 \times #4]}}
\newcommand{\myPsiSubAppend}[5]{\vec{\Psi}_{[#1,\;#2,\;#3,\;#4 \times #5]}}

\newcommand{\layer}{\mathtt{layer}}

\renewcommand{\varepsilon}{\mathcal{E}}
\newcommand{\myEpsilon}{\boldsymbol{\varepsilon}}

\newcommand{\myEpsilonSub}[2]{\boldsymbol{\varepsilon}_{[#1,\;#2]}}

\newcommand{\LSEM}{\operatorname{LSEM}}
\newcommand{\SPAN}{\operatorname{SPAN}}
\newcommand{\Rel}[2]{\operatorname{Rel}(#1, #2)}
\newcommand{\dbound}{\Theta\left(k^8 \log^4(n) \right)}
\newcommand{\Tr}{\operatorname{Tr}}
\newcommand{\range}{\operatorname{range}}

\begin{document}

\title{Robust Identifiability in Linear Structural Equation Models of Causal Inference}

\author{ Karthik Abinav Sankararaman \and
Anand Louis \and
Navin Goyal
}
\address{
University of Maryland, College Park\and
Indian Institute of Science, Bangalore \and
Microsoft Research, India
} 
\email{
karthikabinavs@gmail.com \and
anandl@iisc.ac.in \and
navingo@microsoft.com
}

\maketitle

\begin{abstract}
In this work, we consider the problem of robust parameter estimation from observational data in the context of linear structural equation models (LSEMs). 
LSEMs are a popular and well-studied class of models for inferring causality in the natural and social sciences. One of the main problems related to LSEMs is to recover the model parameters from the observational data. 
Under various conditions on LSEMs and the model parameters the prior work provides efficient algorithms to recover the parameters. However, these results are often about generic identifiability. In practice, generic identifiability is not sufficient and we need robust identifiability: small changes in the observational data should not affect the parameters by a huge amount. Robust identifiability has received far less attention and remains poorly understood.  Sankararaman et al. (2019) recently provided a set of sufficient conditions on parameters under which robust identifiability is feasible. However, a limitation of their work is that their results only apply to a small sub-class of LSEMs, called ``bow-free paths.'' In this work, we significantly extend their work along multiple dimensions. First, for a large and well-studied class of LSEMs, namely ``bow free'' models, we provide a sufficient condition on model parameters under which robust identifiability holds, thereby removing the restriction of paths required by prior work. We then show that this sufficient condition holds with high probability which implies that for a large set of parameters robust identifiability holds and that for such parameters, existing algorithms already achieve robust identifiability. Finally, we validate our results on both simulated and real-world datasets.
\end{abstract}

\section{Introduction}
\label{sc:intro}
	Causal inference is a central problem in a variety of fields in the natural and social sciences. The goal of causal inference is to design methodologies that infer if a group of events \emph{cause} a particular phenomenon or not. A canonical example is the age-old debate on whether smoking causes cancer (\cite{108}). The causal inference problem has been extensively studied in statistics, economics, epidemiology, computer science among others (\eg \cite{64,67,pearlBook,95, PearlMackenzie18}) and several schools of thought exist. One important and popular model is the \emph{linear structural equation} model ($\LSEM$); see, \eg \cite{25} and \cite{Bollen}. Informally, the experimenter has a model of the world and a dataset (represented as samples from a latent distribution) collected during the experiment. The goal is to use the samples and the model to infer the strength of dependencies between various quantities of interest. In $\LSEM$, the experimenter's model is a Gaussian linear model which is formally defined as follows. 
	
			\begin{figure}
			\centering
			\includegraphics[scale=0.3]{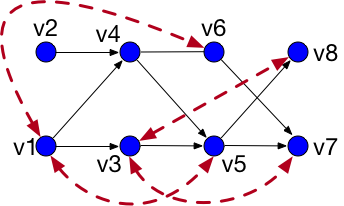}
				\caption{Illustration of a $2$-bow-free graph where the maximum in-degree and out-degree in any vertex is $2$. Black solid lines represent causal edges and red dotted lines represent correlation of the noise parameters.}
				\label{fig:kLayer}
			\end{figure}

		The model of the causal relationship is given by a mixed graph $G=(V, E, F)$, where the vertex set $V$ of size $n$ corresponds to the set of observable random variables. Let $\vec{X} \in \mathbb{R}^{n \times 1}$ denote the vector of random variables corresponding to the vertices in $V$.
		The set $E$ of \emph{directed} edges captures the direction of causality in the model: an edge from vertex $u$ to vertex $v$ implies that 
		$\vec{X}_u$ causes $\vec{X}_v$. We will assume that the edges in $E$ form an acyclic directed graph. The set $F$ of \emph{bidirected} edges denotes the presence of confounding effects (described shortly). Let $\eta \in \mathbb{R}^{n \times 1}$ denote a vector of noise random variables whose covariance matrix is given by $\myOmega\in \mathbb{R}^{n \times n}$. We assume that $\eta$ is a zero-mean multivariate Gaussian random variable. Let $\myLambda \in \mathbb{R}^{n \times n}$ denote the matrix of edge weights on the directed edges; the entries in $\myLambda$ can be interpreted as encoding the strength of causal influence.
				
		The $\LSEM$ model posits that the dependencies between observed variables are linear: the effect on a particular random variable $\vec{X}_u$ is jointly determined by its immediate parents in the directed component of the graph plus a Gaussian noise ($\eta_u$), which we can represent as
			\begin{equation}
		\label{eq:LinearSEM}
		\textstyle \mathbf{X} = \mathbf{\Lambda}^T \mathbf{X} + \mathbf{\eta}.
		\end{equation}
		

		The edge set $E$ puts constraint on the zero pattern of $\myLambda$: if $(u,v) \notin E$, then $\myLambda_{(u, v)} = 0$. Let us denote the set of such matrices by $W(E)$.   
		The bidirected edge set
		$F$ specifies the zero pattern of $\myOmega$: if $u \neq v$ and $(u,v) \notin F$, then $\myOmega_{(u,v)} = 0$. Let $PD(F)$ denote the set of positive semidefinite matrices satisfying this constraint, and let $PD$ be the set of positive semidefinite matrices whose dimensions will be clear from the context. We assume that the dataset is sampled from a distribution that is unknown to the experimenter and has the following properties. 
		
			Since the random vector $\eta$ is a Gaussian random variable with mean zero, it follows that $\vec{X}$ is also a Gaussian random variable with mean zero. Thus, the tuple $(\myLambda, \myOmega)$ defines the distribution of $\mathbf{X}$. We are interested in this map and its invertibility.
			Since $\mathbf{X}$ is Gaussian, instead of working with its distribution we can work with its covariance matrix which is a sufficient statistic. This is what we will do in the sequel. 
			Let $\mySigma$ denote the covariance matrix of $\vec{X}$ and let $\Phi_G: (\myLambda, \myOmega) \mapsto \mySigma$ be the map of interest.
			From the linear relationship in Eq.~\eqref{eq:LinearSEM} we have
				\begin{equation}
					\label{eq:SigmaLSEM}
					\textstyle 	\mySigma = (\Identity - \myLambda)^{-T} \myOmega (\Identity - \myLambda)^{-1}.	
				\end{equation}
				Hence, the map $\Phi_G: W(E) \times PD(F) \to PD$ is given by 
				\begin{align*}
				\textstyle \Phi_G: (\myLambda, \myOmega) \mapsto (\Identity - \myLambda)^{-T} \myOmega (\Identity - \myLambda)^{-1}.
				\end{align*}
				The (global) identifiability question for $G$, namely are the parameters $(\myLambda, \myOmega)$ recoverable from $\mySigma$ for all $\mySigma \in PD$, has a positive answer iff $\Phi_G$ is invertible. 
				The class of mixed graphs $G$ for which $\Phi_G$ is invertible has been precisely characterized by \cite{51}. But this turns out
				to be too strong a restriction and a slightly weaker notion of \emph{generic identifiability} is considered. A mixed graph $G$ is said to be
				generically identifiable if for almost all 
				$(\myLambda, \myOmega) \in W(E) \times PD(F)$, we can recover these parameters from $\Phi_G(\myLambda, \myOmega)$. Here ``almost all'' is meant in the measure-theoretic sense for any reasonable measure such as the Lebesgue or Gaussian measure on $W(E) \times PD(F)$. 
				
				A central question in the study of $\LSEM$s is determining if a mixed graph is \emph{generically identifiable} (GI) and estimating the parameters from the covariance matrix when GI does hold. 
				While mixed graphs for which GI holds have not been completely characterized, many classes of such graphs have been found, (\eg 
				\cite{BP2006UAI,DW2016Scandinavian,51,FDD2012Annals,85}). In particular, \emph{bow-free} graphs (\cite{BP2006UAI}) form one such class and will be studied in this paper. For this class, we can first compute the matrix $\myLambda$ from the covariance matrix $\mySigma$ and then recover $\myOmega$ by computing $(\Identity-\myLambda)^T \mySigma 	(\Identity-\myLambda)$. Since this does not involve matrix inversion, this can be done in a robust manner. Note that this $\myOmega$ may not satisfy the zero-patterns mandated by the model; however, this can be remedied by solving the convex optimization problem for finding the closest PSD matrix satisfying the required zero-pattern. Triangle inequality implies that the optimal solution to the convex optimization problem is a PSD matrix that is also close to the original $\myOmega$ with the same zero-patterns. Thus, we will be primarily interested in the inverse map
					\begin{equation}
						\label{eq:inverseProblem}
						\Psi_{G}^{-1}: \mySigma \rightarrow \myLambda.
					\end{equation}
				Much of the prior work has focused on designing algorithms with the assumption that the \emph{exact} joint distribution over the variables is available, which in turn gives exact $\mySigma$. However, in practice, the data is noisy and inaccurate and the joint distribution is generated via \emph{finitely many} samples from this noisy data. This leads to the question of (generic) \emph{robust identifiability} (RI): if $\mySigma$ is perturbed
				slightly, does $\Psi_{G}^{-1}(\mySigma)$ change only slightly? We will formalize this notion in terms of the condition number. For parameter estimation algorithms to be useful we need robust identifiability to hold because of unavoidable inaccuracies in the input in practice.\footnote{In fact, \cite{SS2016UAI} and \cite{ourUAI19} construct families of examples where the inaccuracies compound to lead to a large error in the final output in semi-Markovian models and $\LSEM$s respectively.} Motivated by this, the key question we consider in this paper is the following.
				\begin{quote}
					\emph{Are bow-free LSEMs robustly identifiable?} 
				\end{quote}
			\xhdr{Our contributions and discussion.} We show that if the parameters satisfy a certain condition then robust identifiability holds for acyclic graphs and that it can be achieved using the algorithm in \cite{FDD2012Annals}. In particular, we show our results for any \emph{bow-free} causal model as long as the covariance matrix satisfies a sufficient condition. We do so by parameterizing the bow-free graph by $k$ which is the maximum indegree and outdegree of any vertex in the graph (this is a standard way to parametrize directed graphs). Moreover, we prove that when the model parameters $(\myLambda, \myOmega)$ are drawn from a suitable random generative process, then the sufficient condition holds with high-probability. We corroborate our theoretical analysis with simulations on a gene expression dataset used in \cite{RICF} and also on additional simulated datasets. Our paper has both conceptual and technical novelty compared to \cite{ourUAI19}. First, \cite{ourUAI19} analyze the error accumulated on every edge; such a strategy fails for anything beyond paths. Here, we instead analyze the total error accumulated across many edges together. The key challenge is in finding the right set of edges to be grouped. Here we show that we need to analyze the total error in computing the weight parameter of all the incoming edges to a vertex $v$. On the technical side, while we use the same high-level idea of induction, we need to work with matrices instead of scalars. This brings up many new non-trivial challenges requiring matrix-theoretic arguments.
				
				It is occasionally pointed out that the algorithms mentioned above (\eg \cite{FDD2012Annals}) are designed for the purpose of identifiability and not for parameter estimation, and as such should not be used for the latter. While, a priori, this could be true, for the specific case of the above algorithms we do not see any reason for not using them for parameter estimation other than the fact that they assume access to the exact covariance matrix. That the access to the exact covariance matrix is not essential under reasonable conditions on parameters is in fact the main point of our paper. This shows that the algorithms designed assuming exact access \emph{can} be used for parameter estimation in realistic situations.
				It's also pertinent to note here that the field of robust statistics seeks to deal with similar situations (under various models of perturbations, often adversarial) by designing new algorithms with the explicit goal of robust identifiability (see \cite{Diakonikolas-Kane} and references therin for a recent survey). Our results show that under a reasonable model of perturbation, existing algorithms are already robust. We are not aware of any work on LSEMs in the robust statistics literature. 
				
				A related point is that if one were to not use the above algorithms for parameter estimation then one needs alternative algorithms. Unfortunately, we are not aware of any algorithms with provable guarantees for parameter estimation other than the ones mentioned above---regardless of the access to the covariance matrix being exact or not. RICF algorithm (\cite{RICF}) is designed expressly for parameter estimation using the maximum likelihood principle from finitely many samples. Maximum likelihood based algorithms come equipped with confidence intervals which provide an estimate of uncertainty in parameter estimation and could potentially be useful for our problem. Unfortunately this is not the case: For one, we are not aware of a quantitative analysis using confidence intervals. Second, we allow adversarial perturbations for which confidence intervals are not applicable. Third, while practically useful, RICF does not provide any theoretical guarantees on finding the correct parameters. 
				It only guarantees that the parameters it finds achieve a local maximum of the likelihood (there are empirical indications that under some conditions it does find the global maximum). Thus, there is a need for algorithms for parameter estimation with provable guarantees without assuming exact access to the covariance matrix or the distribution. As already mentioned, in this paper we show that the existing identifiability algorithms are in fact such algorithms under reasonable conditions on parameters. For another discussion of the identifiability vs. estimation issue we refer the reader to a recent manuscript (\cite{Maclaren}), though they do not provide any positive result like ours. 
				
				\paragraph{Related work.} The issue of robust identifiability for causal models has started to gain attention only recently. \cite{SS2016UAI, ourUAI19, Maclaren} are the only papers we are aware of. \cite{SS2016UAI} showed by means of an example that the recovered parameters can be very sensitive to errors in the data and so robust recovery is not always possible. They worked in the setting of semi-Markovian models (see, \eg \cite{Shpitser2008}). Their example is carefully constructed for the purpose of showing that robust recovery is not possible, and it is not clear if such examples are likely to arise in practice. In other words, their result leaves open the possibility that robust recovery may be possible for a large part of parameter space (according to some reasonable probability measure). A result in this direction was provided by \cite{ourUAI19} for a subclass of LSEMs. For bow-free paths they show that if the parameters are chosen from a certain random distribution then the parameters are robustly identifiable with high probability. Our results in the present paper build upon \cite{ourUAI19}. In particular, our Lemma~\ref{lem:MainInduction} generalizes Lemma 1 in \cite{ourUAI19}. Moreover, given this lemma, the proof for the bound on the condition number follows as in prior work.
				Finally, \cite{Maclaren} provide an abstract framework for studying the robust identifiability problems within the context of causal inference. They also relate it to the extensive literature on similar problems in statistics and inverse problems and provide an entry point to this literature.

				Recently, \cite{Ghoshal_AISTATS, Ghoshal_NIPS} gave an algorithm for parameter estimation and structure learning for linear SEMs from observational data with theoretically good sample and computational complexity and under stochastic noise under certain conditions on the parameters. However, they make the strong assumption that the noise covariance matrix $\myOmega$ is diagonal (and in the second paper under the stronger assumption that $\myOmega$ is a multiple of the identity) which may be overly restrictive in many settings (\cite{RICF}). Thus their result is not comparable to ours.  
				
				There is also a significant body of work on problems such as model misspecification. These are related to but are distinct from the problem studied in the present paper. We refer to \cite[Sec. 1.2]{ourUAI19} for references and commentary on the differences. A very recent example in the same vein is (\cite{Pearl_ICML19}). Again, while sharing similar general motivation, this work is complementary to ours.

\section{Preliminaries}

\xhdr{Notation.} Throughout this paper, we use the notation $G=(V, E, F)$ to represent a causal mixed graph structure where $V$ denotes the set of vertices, $E$ denotes the set of directed (causal) edges and $F$ denotes the set of bidirected (covariance of noise) edges. For simplicity, we assume that the vertices in the set $V$ are indexed $\{1, 2, \ldots, |V| \}$. Throughout this paper, we assume that the directed edges $E$ induce an acyclic graph. For a matrix $\vec{A}$, we use the notation $\| \vec{A} \| := \max_{\vec{x} \neq 0} \frac{ \| \vec{A} \vec{x} \|_2 }{ \| \vec{x} \|_2}$ to denote the spectral norm of this matrix. For a vector $\vec{b}$, we denote $\| \vec{b} \| = \sqrt{\vec{b}^T \cdot \vec{b}}$ to be the $2$-norm. We use many standard properties of the spectral norm in the proofs of this paper. Lemma~\ref{lem:normProp} in the appendix summarizes these for completeness. We use $\sigma_1(\vec{A}), \lambda_1(\vec{A})$ to denote the largest singular and eigenvalue respectively, of matrix $\vec{A}$. We let $\myLambdaSub{I}{J}, \myOmegaSub{I}{J}, \mySigmaSub{I}{J} \in \mathbb{R}^{|I| \times |J|}$ to denote the sub-matrix of $\myLambda, \myOmega, \mySigma$ respectively, corresponding to vertices in the index set $I$ and $J$. For two given vertices $u,v \in V$, we use $\myLambda_{u, v}, \myOmega_{u, v}, \mySigma_{u, v}$ to refer to the $(u, v)$-entry of respective matrices. We use $\poly(n)$ to denote a function which is polynomial in $n$. $\mathcal{U}[a, b]$ denotes the uniform distribution on the interval $[a, b]$ with pdf $f(x) = 1/(b-a)$. 

We denote $\layer(i)$ to be the set of vertices such that $v \in \layer(i)$ if and only if the longest directed path ending in $v$ has length $i$. Thus, $\layer(1)$ denotes the set of vertices with no incoming directed edge. For any vertex $v$, we denote $\pav$ to be the set of vertices in $V$ such that there is a directed edge from every vertex in $\pav$ to $v$. Additionally, we use the notation $\spa := \pa \pav$. Since the graph is acyclic, there exists a topological sort order of the vertices $V$ (\cite{FDD2012Annals}). Throughout this paper, we assume that $n$ is an asymptotic parameter; thus $o(1)$ denotes terms that go to $0$ as $n$ goes to infinity.

\begin{definition}[$k$-bow-free causal graphs]
	\label{eq:definition}
	A causal graph $G=(V, E, F)$ is called a $k$-bow-free causal graph if it has the following properties. 
	\begin{enumerate}
		\item \textbf{Bow-free.} The graph is bow-free \ie between any two vertices $u$ and $v$, there is never both a directed and bidirected edge.
		\item \textbf{Maximum in-degree or out-degree of $k$.} For any vertex $v \in V$, the total number of directed edges coming into $v$ is at most $k$. Likewise, the total number of directed edges leaving $v$ is also $k$. Thus, $| \pav | \leq k$ for every $v \in V$.
	\end{enumerate}
\end{definition}

Figure~\ref{fig:kLayer} pictorially denotes an example of $k$-bow-free causal graph. Throughout this paper, $k$ should be viewed as a small constant (for instance in our experiments $k$ is either $2$ or $7$). As in prior work (\cite{ourUAI19, SS2016UAI}), we use the notion of \emph{condition number} to measure the robustness of the models. Before we define the condition number, we define the  relative distance between two matrices. Given matrices $\mathbf{A}, \mathbf{B} \in \mathbb{R}^{n \times m}$, we define the \emph{relative distance}, denoted by $\Rel{\mathbf{A}}{\mathbf{B}}$ as the following: $\Rel{\mathbf{A}}{\mathbf{B}} := \max_{\substack{1 \leq i \leq n,\\ 1 \leq j \leq m :\\ \Abs{A_{i, j}} \neq 0}} \frac{|A_{i, j}-B_{i, j}|}{\Abs{A_{i, j}}}$.
	The $\ell_{\infty}$-condition number is defined as follows.
	\begin{definition}[Relative $\ell_{\infty}$-condition number]
			\label{def:linftyCondition}
			Let $\mathbf{\Sigma}$ be a given data covariance matrix and $\mathbf{\Lambda}$ be the corresponding parameter matrix. Let a $\gamma$-perturbed family of matrices be denoted by $\mathcal{F}_{\gamma}$  (\emph{i.e.,} set of matrices $\mathbf{\tilde{\Sigma}}_{\gamma}$ such that $\Rel{\mathbf{\Sigma}}{\tilde{\mathbf{\Sigma}}_\gamma} \leq \gamma $). For any $\mathbf{\tilde{\Sigma}}_{\gamma} \in \mathcal{F}_{\gamma}$ let the corresponding recovered parameter matrix be denoted by $\mathbf{\tilde{\Lambda}}_{\gamma}$. Then the relative $\ell_{\infty}$-condition number is defined as, 
			\begin{equation}
				\label{eq:linftyCondition}
			\textstyle	\kappa(\mathbf{\Lambda}, \mathbf{\Sigma}) := \sup_{\gamma < \frac{1}{n^4}} \mathtt{ess~sup}_{\tilde{\mathbf{\Sigma}}_\gamma \in \mathcal{F}_{\gamma}} \frac{\Rel{\mathbf{\Lambda}}{\tilde{\mathbf{\Lambda}}_\gamma}}{\Rel{\mathbf{\Sigma}}{\tilde{\mathbf{\Sigma}}_\gamma} }.
			\end{equation}
 	\end{definition}
	
	Condition number as the notion of stability is useful since a bound on this quantity translates to an upper-bound on the sample complexity. More precisely, to get an error of $\epsilon$ in the output polynomial in $1/\epsilon$, condition number and other parameters of the input number of samples suffice (\emph{e.g.,} \cite{srivastava2013covariance}).
	\subsection{Required Background from Prior Work}
We give a self-contained background needed from \cite{FDD2012Annals} for our paper.
	\begin{definition}[Half-trek (\cite{FDD2012Annals})]
		\label{def:halftrek}
		For any given vertex $v \in V$, the set $htr(v)$ denotes the set of vertices that can be reached from $v$ via a path of the form,
		\[
				v \leftrightarrow v_1 \rightarrow v_2 \rightarrow \ldots \rightarrow v_d \quad \text{or}, v \rightarrow v_2 \rightarrow \ldots \rightarrow v_d.
		\]
	\end{definition}

	\begin{definition}[Parameter Recovery Algorithm from \cite{FDD2012Annals}]
		\label{def:FDDAlg}
		Consider a vertex $v \in V$ such that $pa(pa(v)) \neq \phi$. The goal is to compute the vector $\myLambdaSub{pa(v)}{\{v\}}$. Let $Y_v = \{y_1, y_2, \ldots, y_k\}$ be a given set of vertices corresponding to vertex $v$. Let $pa(v) =\{p_1, p_2, \ldots, p_k\}$ denote the set of parents of $v$. Let $\vec{A}$ be a matrix such that $A_{i, j} = [(\Identity - \myLambda)^T \cdot \mySigma]_{y_i, p_j}$ if $y_i \in htr(v)$ and $A_{i, j} = [\mySigma]_{y_i, p_j}$ otherwise. Likewise, let $\vec{b}$ denote a vector such that $b_i =  [(\Identity - \myLambda)^T \cdot \mySigma]_{y_i, v}$ if $y_i \in htv(v)$ and $b_i =  [\mySigma]_{y_i, v}$ otherwise. Then we have,
		\begin{equation}
			\label{eq:FoygelGeneral}
				\myLambdaSub{\pav}{v} = \vec{A}^{-1} \cdot \vec{b}.
		\end{equation}
		For vertices $v \in V$ such that $\pa \pav = \phi$ we compute $\myLambdaSub{pa(v)}{\{v\}}$ using the expression,
		\begin{equation}
			\label{eq:FoygelPart}
				\myLambdaSub{\pav}{v} 	= \mySigmaSub{\pav}{\pav}^{-1} \cdot \mySigmaSub{\pav}{v}.
		\end{equation}
	\end{definition}

\section{Inverse Problem with Adversarial Noise}
\label{sec:adversarial}
	In this section, we consider $\LSEM$s with $k$-bow-free graph and show that under a sufficient condition (formally defined in the assumptions of Model~\ref{mod:mainModel}), these models can be robustly identified using the algorithm in \cite{FDD2012Annals} in the presence of \emph{adversarial} noise. The model we consider is as follows.
	\begin{model}
		\label{mod:mainModel}
	 We consider the following model of perturbation. Assume that we are given a data covariance matrix $\mySigma$. Let $\myEpsilon \in \mathbb{R}^{n \times n}$ denote the matrix of perturbations. Fix a small $0 < \gamma < \frac{1}{n^4}$. Thus, the perturbed matrix is $\tilde{\mySigma} := \mySigma + \myEpsilon$. Additionally, we posit the following property on the perturbation. For every entry $(i, j)$ we have $\varepsilon_{i, j} \leq \frac{\gamma}{\sqrt{k}} \Sigma_{i, j}$. WLOG we assume that there exists an entry $(i, j)$ such that $\varepsilon_{i, j} = \frac{\gamma}{\sqrt{k}} \Sigma_{i, j}$. We have the following assumptions for every vertex $v \in V$.
	 \begin{enumerate}[label=\textbf{(A.\arabic*)},ref=Assumption (A.\arabic*)]
	 	\item \label{mod:condNumberInput} \underline{\textbf{Input Condition Number.}} The condition number of the principal sub-matrix $\mySigmaSub{\pav}{\pav}$, defined as $\kappa(\mySigmaSub{\pav}{\pav}) := \| \mySigmaSub{\pav}{\pav}^{-1} \| \| \mySigmaSub{\pav}{\pav} \| \leq \kappa_0 \leq \frac{1}{2 \gamma}$.
		\item \label{mod:offDiagonal} \underline{\textbf{Diagonal dominance.}}  For some $0 < \alpha < 1$, the following hold:\\ $\| \mySigmaSub{\pav}{v} \| \leq \alpha \| \mySigmaSub{\pav}{\pav} \|$, $\| \mySigmaSub{\spa}{\pav} \| \leq \alpha \| \mySigmaSub{\pav}{\pav} \|$ and\\ $\| \mySigmaSub{\spa}{v} \| \leq \alpha \| \mySigmaSub{\pav}{\pav} \|$. 
	 	\item \label{mod:Lambda} \underline{\textbf{Normalized parameters.}}  We have $\| \myLambdaSub{\spa}{\pav} \| \leq \beta < 1$. Additionally, for every directed edge $(u \rightarrow v)$ in the causal DAG, we have $| \Lambda_{u, v} | \geq \frac{1}{\lambda} > \frac{1}{n^2}$ where $\Lambda_{u, v}$ represents the edge-weight. 
	 \end{enumerate} 
	\end{model}
	
	\xhdr{Intuition on the assumptions.} Before we state our theorem, we provide some intuition on the assumptions. An upper-bound on the condition number of the input matrix (as in \ref{mod:condNumberInput}) is a necessary condition even in the simplest case of robustly solving a system of linear equations. More specifically, the relative error in solving a system of linear equations compared to a perturbed instance is upper-bounded by the condition number of the constraint matrix (Example 3.4 in \cite{stewart1998matrix}). Since $\LSEM$s significantly generalize this, it is natural that such a condition should be \emph{necessary}. \ref{mod:Lambda} states that $\| \myLambda \|$ corresponding to all incoming edges for any set of vertices $\pav$ is upper-bounded by a constant less than $1$. Intuitively, it means that the total ``information'' passed from the vertices appearing earlier in the topological order to those in the later parts does not blow up. The \emph{a priori} limiting assumption is \ref{mod:offDiagonal}; this is required for technical reasons to make the analysis go through. Intuitively, this assumption is a version of the \emph{diagonal-dominance} in matrices; however, we require a comparison between a principal sub-matrix and a neighboring $k$-dimensional sub-matrix. We show in Section~\ref{sec:randomModel} that under an arguably natural generative model for $\LSEM$s, with high probability the generated $\LSEM$ satisfies \ref{mod:offDiagonal} suggesting that it is in fact not a strong assumption. 
	
	The main result of the paper is the following bound on the $\ell_{\infty}$-condition number of any bow-free $\LSEM$ satisfying the assumptions in Model~\ref{mod:mainModel}.
	
\begin{btheorem}
	\label{thm:MainTheorem}
	Consider a $k$-bow-free causal model denoted by the mixed graph $G=(V, E, F)$. If $\alpha \beta \kappa_0 < 0.99$ and $\frac{\alpha \kappa_0}{1-\alpha \beta \kappa_0} \left( 1 + \frac{\kappa_0 (1+\beta)}{1-\alpha \beta \kappa_0} \right) < \frac{0.99}{k}$ then for the model of perturbations described in Model~\ref{mod:mainModel} we have that the condition number $\kappa(\mathbf{\Lambda}, \mathbf{\Sigma}) \leq \mathcal{O}\left( \frac{n^2}{\sqrt{k}} \right)$. 
\end{btheorem}

To prove the main theorem, we first show the following lemma which bounds the difference between the true and the recovered parameter. 
	\begin{restatable}[]{lemma}{mainLemma}
	\label{lem:MainInduction}
	If $\alpha \beta \kappa_0 < 1$ and $\frac{\alpha \kappa_0}{1-\alpha \beta \kappa_0} \left( 1 + \frac{\kappa_0 (1+\beta)}{1-\alpha \beta \kappa_0} \right) < \frac{0.99}{k}$ then for every $v \in \layer(j)$ and every $j \geq 2$  we have, $\| \myLambdaSub{\pav}{v} - \myLambdaTSub{\pav}{v} \| \leq \eta \cdot \gamma$, where $\eta$ is the following depending on the parameters in Model~\ref{mod:mainModel}
	\begin{multline*}
		\textstyle \eta := 
		10 \ast \left( \frac{\alpha \kappa_0^2 (1+\beta)(1+\beta+o(1) )}{(1- \alpha \beta \kappa_0)^2} + \frac{\kappa_0 \alpha (1+\beta +o(1))}{1-\alpha \beta \kappa_0} \right) \cdot
		 \left( 1 - \frac{\alpha \kappa_0}{1-\alpha \beta \kappa_0} - \frac{\alpha \kappa_0^2 (1+\beta)}{(1-\alpha \beta \kappa_0)^2} \right)^{-1} + o(1).
	\end{multline*}
	\end{restatable}
	
	\xhdr{Proof outline.} 
	At a high level, our proof strategy is similar in spirit to that of \cite{ourUAI19}; they prove an analogous result for graphs that are paths (for a model similar to Model~\ref{mod:mainModel}). However, since we prove such a result for general graphs, our setting faces many additional technical challenges. Similar to  \cite{ourUAI19}, we prove the main technical Lemma \ref{lem:MainInduction}, using induction over the layers. For any vertex $v$, we can compute $\myLambdaSub{\pav}{v}$ using equations \ref{eq:FoygelGeneral} and \ref{eq:FoygelPart}. Using the induction hypothesis, we get that $\myLambda$ for the previously considered layers has a sufficiently ``small'' error. Let $\vec{A}_v$ and $\vec{b}_v$ denote $\vec{A}$ and $\vec{b}$ from equation \ref{eq:FoygelGeneral} for vertex $v$ when working with the true (unperturbed) $\Sigma$, and let $\tilde{\vec{A}}_v$ and $\tilde{\vec{b}}_v$ denote the corresponding matrices for $\tilde{\Sigma}$. We show that the spectral norm of the matrix $\tilde{\vec{A}}_v - \vec{A}_v$ and the norm of the vector $\tilde{\vec{b}}_v - \vec{b}_v$ is sufficiently small. We use this and bounds on the norms of $\vec{A}_v$ and $\vec{b}_v$ to show that the norm of $\tilde{\vec{A}}_v^{-1} \tilde{\vec{b}}_v - \vec{A}_v^{-1} \vec{b}_v $ is small.
These steps pose multiple subtle technical challenges in comparison to \cite{ourUAI19}, and require new ideas to handle them.

\xhdr{Proof of Theorem~\ref{thm:MainTheorem}.} From Lemma~\ref{lem:MainInduction} we have that $\| \myLambdaSub{\pav}{v} - \myLambdaTSub{\pav}{v} \| \leq \eta \gamma$. From \ref{prop:5} in Lemma~\ref{lem:normProp} we have that the absolute value of every entry in the matrix $(\myLambdaSub{\pav}{v} - \myLambdaTSub{\pav}{v})$ is at most $\| \myLambdaSub{\pav}{v} - \myLambdaTSub{\pav}{v} \| \leq \eta \gamma$. Combining this with \ref{mod:Lambda} we have $\Rel{\myLambda}{\tilde{\myLambda}} \leq \frac{\eta \gamma}{\lambda}$. Moreover, from Model~\ref{mod:mainModel} we have that $\Rel{\mySigma}{\tilde{\mySigma}} = \frac{\gamma}{\sqrt{k}}$. Thus, we get that the condition number is at most $\kappa(\myLambda, \mySigma) \leq \frac{\eta \sqrt{k}}{\lambda}.$ From \ref{mod:Lambda}, we have $\frac{1}{\lambda} \geq \frac{1}{n^2}$ which implies that $\kappa(\myLambda, \mySigma) \leq \eta \sqrt{k} n^2$. From the definition of $\eta$ and the premise of Theorem~\ref{thm:MainTheorem} we have that $\eta \leq \mathcal{O}\left( \frac{1}{k} \right)$. Thus, we get the stated bound.

\subsection{Proof of Lemma~\ref{lem:MainInduction}} We use the following additional notations in the proof of Lemma~\ref{lem:MainInduction}. For any matrix $\myOmega$ and $I \subseteq [n]$, we denote $\myOmegaSub{I}{\ast}$ to denote the rows corresponding to the indices $I$ in $\myOmega$. For $i \geq 3$, we define recursively as $\pah{i}{v} := \pah{i-1}{pa(v)}$ and $\pah{2}{v} := \spa$. We now define notations that hold for every for every $v \in \layer(j)$ where $j \geq 3$. Define $\vec{A}_{v} := \mySigmaSub{\pav}{\pav} - \myLambdaSub{\spa}{\pav}^T \cdot \mySigmaSub{\spa}{\pav}$ and $\vec{b}_{v} := \mySigmaSub{\pav}{v} - \myLambdaSub{\spa}{\pav}^T \mySigmaSub{\spa}{v}$. Likewise, define $\tilde{\vec{A}}_{v} := \mySigmaTSub{\pav}{\pav} - \myLambdaTSub{\spa}{\pav}^T \cdot \mySigmaTSub{\spa}{\pav}$ and $\tilde{\vec{b}}_{v} :=  \mySigmaTSub{\pav}{v} - \myLambdaTSub{\spa}{\pav}^T \mySigmaTSub{\spa}{v}$. Define $\vec{\Delta_{A_{v}}} := \tilde{\vec{A}}_{v} - \vec{A}_{v}$ and $\vec{\Delta_{b_{v}}} := \tilde{\vec{b}}_{v} - \vec{b}_{v}$. For any $v \in \layer(2)$ we define the following. Define $\vec{A}_{v} := \mySigmaSub{\pav}{\pav}$ and $\vec{b}_{v} := \mySigmaSub{\pav}{v}$. 

Before we start the proof of Lemma~\ref{lem:MainInduction} we state the following claims. These follow from the assumptions in Model~\ref{mod:mainModel}.
	\begin{restatable}{claim}{ErrorVsSigma}
		\label{obs:ErrorVsSigma}
		For every $v \in \layer(j)$ and every $j \geq 3$, we have $\| \myEpsilonSub{\pav}{\pav} \| \leq \gamma \| \mySigmaSub{\pav}{\pav} \|$, $\| \myEpsilonSub{\pav}{v} \| \leq \gamma \alpha \| \mySigmaSub{\pav}{\pav} \|$, $\| \myEpsilonSub{\spa}{\pav} \| \leq \gamma \alpha \| \mySigmaSub{\pav}{\pav} \|$ and $\| \myEpsilonSub{\spa}{v} \| \leq \gamma \alpha \| \mySigmaSub{\pav}{\pav} \|$.
	\end{restatable}
			\begin{proof}
		Let $I, J \subseteq [n]$ be arbitrary subsets such that $|I| = |J| = k$. Consider $\| \myEpsilonSub{I}{J} \|$. From \ref{prop:5} this is $\sqrt{\lambda_1(\myEpsilonSub{I}{J}^T \cdot \myEpsilonSub{I}{J})}$. From \ref{prop:7} we have,
		\begin{align}
			 \sqrt{\lambda_1(\myEpsilonSub{I}{J}^T \cdot \myEpsilonSub{I}{J})} & \leq \sqrt{\Tr[\myEpsilonSub{I}{J}^T \cdot \myEpsilonSub{I}{J}]} \nonumber \\
			 & \leq \sqrt{\frac{\gamma^2}{k} \Tr[\mySigmaSub{I}{J}^T \cdot \mySigmaSub{I}{J}]} \quad \text{(Using Model~\ref{mod:mainModel})} \nonumber \\
			 & \leq \frac{\gamma}{\sqrt{k}} \sqrt{k} \| \mySigmaSub{I}{J} \| \quad \text{(Using \ref{prop:8})} \nonumber \\
			 & \leq \gamma \| \mySigmaSub{I}{J} \|. \label{eq:Claim2Main}
		\end{align}
		Letting $I= \pav$ and $J=\pav$ from Eq.~\eqref{eq:Claim2Main} we get $\| \myEpsilonSub{\pav}{\pav} \| \leq \gamma \| \mySigmaSub{\pav}{\pav} \|$. Likewise, letting $I= \spa$ and $J=\pav$ we get $\| \myEpsilonSub{\spa}{\pav} \| \leq \gamma \| \mySigmaSub{\spa}{\pav} \|$. Using \ref{mod:offDiagonal} we get $\gamma \| \mySigmaSub{\spa}{\pav} \| \leq \gamma \alpha \| \mySigmaSub{\pav}{\pav} \|$.
		
		Let $U \subseteq [n]$ be an arbitrary subset such that $|U| =k$. Let $i \in [n]$ be an arbitrary index. Consider $\| \myEpsilonSub{U}{i} \|$. By definition, we have,
		\begin{align*}
			\| \myEpsilonSub{U}{i} \| &= \sqrt{\sum_{u \in U} \myEpsilon_{[u, i]}^2} & \text{Definition of $\ell_2$-norm} \\
			& \leq \sqrt{\sum_{u \in U} \frac{\gamma^2}{k} \mySigma_{[u, i]}^2} & \text{From Model~\ref{mod:mainModel}} \\
			& \leq \gamma \| \mySigmaSub{U}{i} \| & \text{$k \geq 1$} \\
		\end{align*}
		Letting $U = \spa$ and $i=v$ we get $\| \myEpsilonSub{\spa}{v} \| \leq \gamma \| \mySigmaSub{\spa}{v} \|$. From \ref{mod:offDiagonal} we get $\gamma \| \mySigmaSub{\spa}{v} \| \leq \gamma \alpha  \| \mySigmaSub{\pav}{\pav} \|$.
		\end{proof}

	\mainLemma*
	\begin{proof}
	Let $\tau := k \eta n^{-2}$. We proceed by inductively showing the following for every vertex $v \in V$.
	\begin{align}
		& \| \myLambdaSub{\pav}{v} - \myLambdaTSub{\pav}{v} \| \leq \eta \cdot \gamma \label{eq:LambdaInd} \\
		& \| \vec{\Delta_{A_{v}}} \| \leq \gamma (\eta + 1 + \beta + \tau) \|\mySigmaSub{\pav}{\pav}\| \\
		& \| \vec{\Delta_{B_{v}}} \| \leq \gamma \alpha (\eta + 1 + \beta + \tau) \|\mySigmaSub{\pav}{\pav}\|
	\end{align}

	The base-case is for the vertex $v$ such that $\operatorname{pa}(\pav) = \phi$. 

	Consider $\vec{\Delta_{A_{v}}} := \tilde{\vec{A}}_{v} - \vec{A}_{v}$. This can be written as $\mySigmaTSub{\pav}{\pav} - \mySigmaSub{\pav}{\pav} = \myEpsilonSub{\pav}{\pav}$. Thus using Claim~\ref{obs:ErrorVsSigma} we have,
 		\[ \| \vec{\Delta_{A_{v}}} \| = \| \myEpsilonSub{\pav}{\pav} \| \leq \gamma \|\mySigmaSub{\pav}{\pav}\| \leq \gamma (\eta + 1 + \beta + \tau)  \|\mySigmaSub{\pav}{\pav}\|. \] 
 	
 	Likewise consider $\vec{\Delta_{b_{v}}}$ which from Model~\ref{mod:mainModel}, is equal to $\mySigmaTSub{\pav}{v} - \mySigmaSub{\pav}{v} = \myEpsilonSub{\pav}{v}$. Thus from Claim~\ref{obs:ErrorVsSigma} we have,
 			\[ \|\vec{\Delta_{b_{v}}} \| = \| \myEpsilonSub{\pav}{v} \| \leq \gamma \alpha \| \mySigmaSub{\pav}{\pav} \|  \leq  \gamma \alpha (\eta + 1 + \beta + \tau) \|\mySigmaSub{\pav}{\pav}\|. \]
	
	Consider $\| \myLambdaSub{\pav}{v} - \myLambdaTSub{\pav}{v} \|$. From Eq.~\eqref{eq:FoygelPart} we have $\myLambdaSub{\pav}{v} = \mySigmaSub{\pav}{\pav}^{-1} \cdot \mySigmaSub{\pav}{v}$. Thus, we want to show that 
	\[
			\| \mySigmaSub{\pav}{\pav}^{-1} \cdot \mySigmaSub{\pav}{v} - \mySigmaTSub{\pav}{\pav}^{-1} \cdot \mySigmaTSub{\pav}{v} \| \leq \eta \cdot \gamma.
	\]
	By definition we have, 
	\begin{multline*}
		\mySigmaSub{\pav}{\pav}^{-1} \cdot \mySigmaSub{\pav}{v} - \mySigmaTSub{\pav}{\pav}^{-1} \cdot \mySigmaTSub{\pav}{v} =\\
		 \mySigmaSub{\pav}{\pav}^{-1} \cdot \mySigmaSub{\pav}{v} - (\mySigmaSub{\pav}{\pav} + \myEpsilonSub{\pav}{\pav})^{-1} \cdot (\mySigmaSub{\pav}{v} + \myEpsilonSub{\pav}{v}).	
	\end{multline*}
	
	This can be written as 
	\begin{multline}
		\label{eq:step2-new}
		 \mySigmaSub{\pav}{\pav}^{-1} \cdot \mySigmaSub{\pav}{v} - \mySigmaTSub{\pav}{\pav}^{-1} \cdot \mySigmaTSub{\pav}{v} =\\
		\left[ \mySigmaSub{\pav}{\pav}^{-1} - (\mySigmaSub{\pav}{\pav} + \myEpsilonSub{\pav}{\pav})^{-1} \right] \cdot \mySigmaSub{\pav}{v} \\ -  (\mySigmaSub{\pav}{\pav} + \myEpsilonSub{\pav}{\pav})^{-1} \cdot \myEpsilonSub{\pav}{v}.	
	\end{multline}
	Taking spectral norm on both sides of Eq.~\eqref{eq:step2-new} and using \ref{prop:1} and \ref{prop:2} we get,
	\begin{align}
		& \|  \mySigmaSub{\pav}{\pav}^{-1} \cdot \mySigmaSub{\pav}{v} - \mySigmaTSub{\pav}{\pav}^{-1} \cdot \mySigmaTSub{\pav}{v} \| \nonumber \\
		& \leq \| \mySigmaSub{\pav}{\pav}^{-1} - (\mySigmaSub{\pav}{\pav} + \myEpsilonSub{\pav}{\pav})^{-1} \| \| \mySigmaSub{\pav}{v}\| \nonumber \\
		& + \|( \mySigmaSub{\pav}{\pav} + \myEpsilonSub{\pav}{\pav})^{-1}\| \| \myEpsilonSub{\pav}{v}\|. \label{eq:lastLine-new}
	\end{align}
	We will first upper-bound $\| \mySigmaSub{\pav}{\pav}^{-1} - (\mySigmaSub{\pav}{\pav} + \myEpsilonSub{\pav}{\pav})^{-1} \|$ by showing that \\
	$\| \mySigmaSub{\pav}{\pav}^{-1} \myEpsilonSub{\pav}{\pav}\|$ satisfies the premise in Lemma~\ref{lem:matrixPerturbations}. In particular we will show that,
		\begin{equation}
			\label{eq:precondition-new}
			\| \mySigmaSub{\pav}{\pav}^{-1} \myEpsilonSub{\pav}{\pav}\| < \frac{1}{2}.	
		\end{equation}
		From \ref{prop:1} we have $\| \mySigmaSub{\pav}{\pav}^{-1} \myEpsilonSub{\pav}{\pav}\| \leq \| \mySigmaSub{\pav}{\pav}^{-1} \| \| \myEpsilonSub{\pav}{\pav}\|$.
	From Claim~\ref{obs:ErrorVsSigma}, the RHS can be upper-bounded by,
	\[
			\| \mySigmaSub{\pav}{\pav}^{-1} \| \| \myEpsilonSub{\pav}{\pav}\| \leq \gamma \| \mySigmaSub{\pav}{\pav}^{-1} \| \| \mySigmaSub{\pav}{\pav} \|. 
	\]
	From \ref{mod:condNumberInput} we have the RHS is upper-bounded by $\gamma \kappa_0$. From \ref{mod:condNumberInput} we have that $\gamma \kappa_0 \leq \frac{1}{2}$ . Thus,
	\begin{equation}
		\label{eq:preconditionReuse-new}
		\gamma \| \mySigmaSub{\pav}{\pav}^{-1} \| \| \mySigmaSub{\pav}{\pav} \| \leq \frac{1}{2}.	
	\end{equation}
	
	Since Eq.~\eqref{eq:precondition-new} satisfies the premise of Lemma~\ref{lem:matrixPerturbations} we have that,
	\begin{multline}
	\label{eq:tempBaseCase-new}
		 \| \mySigmaSub{\pav}{\pav}^{-1} - (\mySigmaSub{\pav}{\pav} + \myEpsilonSub{\pav}{\pav})^{-1} \|  \\
		 \leq \|\mySigmaSub{\pav}{\pav}^{-1} \| \| \mySigmaSub{\pav}{\pav}^{-1} \myEpsilonSub{\pav}{\pav} \|(1+2 \|  \mySigmaSub{\pav}{\pav}^{-1} \myEpsilonSub{\pav}{\pav} \|). 
	\end{multline}
	Using Claim~\ref{obs:ErrorVsSigma}, Eq.~\eqref{eq:tempBaseCase-new} and \ref{mod:offDiagonal} the first summand in Eq.~\eqref{eq:lastLine-new} can be upper-bounded by,
	\begin{multline}
	\label{eq:firstSummandBaseCase-new}	
	 \| \mySigmaSub{\pav}{\pav}^{-1} - (\mySigmaSub{\pav}{\pav} + \myEpsilonSub{\pav}{\pav})^{-1} \| \| \mySigmaSub{\pav}{v}\|  \\
	 \leq \gamma \alpha \|\mySigmaSub{\pav}{\pav}^{-1} \|^2 \|\mySigmaSub{\pav}{\pav} \|^2  + 2 \gamma^2 \alpha \|\mySigmaSub{\pav}{\pav}^{-1} \|^3 \|\mySigmaSub{\pav}{\pav} \|^3.
	\end{multline}
	Now consider $\|( \mySigmaSub{\pav}{\pav} + \myEpsilonSub{\pav}{\pav})^{-1}\| \| \myEpsilonSub{\pav}{v}\|$. Consider $\| \mySigmaSub{\pav}{\pav}^{-1} \myEpsilonSub{\pav}{v}\|$. We will now show that it satisfies the premise in Lemma~\ref{lem:matrixApprox}, that 
		\begin{equation}
			\label{eq:precondition2-new}
			\| \mySigmaSub{\pav}{\pav}^{-1} \myEpsilonSub{\pav}{v}\| \leq \frac{1}{2}.	
		\end{equation}
		From Claim~\ref{obs:ErrorVsSigma}, $\| \mySigmaSub{\pav}{\pav}^{-1} \myEpsilonSub{\pav}{v}\| \leq \gamma \alpha \| \mySigmaSub{\pav}{\pav}^{-1} \| \| \mySigmaSub{\pav}{\pav} \|$. From Eq.~\eqref{eq:preconditionReuse-new} and the fact that $\alpha < 1$ this is at most $\frac{1}{2}$.

	Since Eq.~\eqref{eq:precondition2-new} satisfies the premise in Lemma~\ref{lem:matrixApprox} we get,
	\begin{align}
		& \|( \mySigmaSub{\pav}{\pav} + \myEpsilonSub{\pav}{\pav})^{-1}\| \| \myEpsilonSub{\pav}{v}\| \nonumber \\
		& \leq \|\mySigmaSub{\pav}{\pav}^{-1} \|(1 + 2 \|\mySigmaSub{\pav}{\pav}^{-1} \myEpsilonSub{\pav}{\pav} \|)\| \myEpsilonSub{\pav}{v}\|. \label{eq:secondSummandBaseCase1-new} 
	\end{align}
	From Claim~\ref{obs:ErrorVsSigma} we have 
	\[ \|\mySigmaSub{\pav}{\pav}^{-1} \| \| \myEpsilonSub{\pav}{v}\| \leq \gamma \alpha \|\mySigmaSub{\pav}{\pav}^{-1} \| \|\mySigmaSub{\pav}{\pav}\|. \] Similarly we have,
	\begin{multline*}
		2  \|\mySigmaSub{\pav}{\pav}^{-1} \| \|\mySigmaSub{\pav}{\pav}^{-1} \|  \| \myEpsilonSub{\pav}{\pav} \| \| \myEpsilonSub{\pav}{v}\| \\ \leq 2 \gamma^2 \alpha \| \mySigmaSub{\pav}{\pav}^{-1} \| \| \mySigmaSub{\pav}{\pav}^{-1} \| \| \mySigmaSub{\pav}{\pav} \| \mySigmaSub{\pav}{\pav} \|.
	\end{multline*}
	Hence, Eq.~\eqref{eq:secondSummandBaseCase1-new} is upper-bounded by,
	\begin{align}
		& \leq \gamma \alpha \|\mySigmaSub{\pav}{\pav}^{-1} \|\|\mySigmaSub{\pav}{\pav} \| + 2 \gamma^2 \alpha \| \mySigmaSub{\pav}{\pav}^{-1} \|^2 \| \mySigmaSub{\pav}{\pav} \|^2 \label{eq:secondSummandBaseCase-new}.
	\end{align}
 	Plugging Eq.~\eqref{eq:firstSummandBaseCase-new} and Eq.~\eqref{eq:secondSummandBaseCase-new} back into Eq.~\eqref{eq:lastLine-new} we get,
	\begin{align*}
			& \| \mySigmaSub{\pav}{\pav}^{-1} \cdot \mySigmaSub{\pav}{v} - \mySigmaTSub{\pav}{\pav}^{-1} \cdot \mySigmaTSub{\pav}{v} \| \\
			& \leq \gamma \alpha \| \mySigmaSub{\pav}{\pav}^{-1} \|^2 \| \mySigmaSub{\pav}{\pav} \|^2 + 2 \gamma^2 \alpha  \| \mySigmaSub{\pav}{\pav}^{-1} \|^3 \| \mySigmaSub{\pav}{\pav} \|^3\\
			& + \gamma \alpha \| \mySigmaSub{\pav}{\pav}^{-1} \| \| \mySigmaSub{\pav}{\pav} \| + 2 \gamma^2 \alpha  \| \mySigmaSub{\pav}{\pav}^{-1} \|^2 \| \mySigmaSub{\pav}{\pav} \|^2 \\
			& \leq \alpha \gamma \kappa_0^2 + 2 \gamma^2 \alpha \kappa_0^3 + \gamma \alpha \kappa_0 + 2 \gamma^2 \alpha \kappa_0^2. \quad \text{(Using \ref{mod:condNumberInput})}\\
			& \leq \alpha \left( \gamma \kappa_0^2 + 2 \gamma^2 \kappa_0^3 + \gamma \kappa_0 + 2 \gamma^2\kappa_0^2 \right)  \\
			& \leq 4 \alpha \kappa_0^2 \gamma \qquad \text{(Assuming $\kappa_0 \geq 1$, $2 \gamma \leq 1$ and $2 \gamma \kappa_0 \leq 1$.)} \\
			& \leq \eta \gamma 
	\end{align*}
	
	We will now consider the inductive case. Let the statement be true for every vertex $v \in \layer(j)$ for $j=2, 3, \ldots, i-1$. Now consider any vertex $v \in \layer(i)$.
	
	We can bound $\| \vec{A}_{v}^{-1} \|$ as follows. We will bound $\| (\mySigmaSub{\pav}{\pav} - \myLambdaSub{\spa}{\pav}^T \cdot \mySigmaSub{\pav}{\pav})^{-1} \|$. We will show that $\| \mySigmaSub{\pav}{\pav}^{-1} \cdot \myLambdaSub{\spa}{\pav}^T \cdot \mySigmaSub{\spa}{\pav} \|$ satisfies the premise in Lemma~\ref{lem:normDifference}. Using \ref{prop:1} this can be upper-bounded by,
	\[
			\| \mySigmaSub{\pav}{\pav}^{-1} \cdot \myLambdaSub{\spa}{\pav}^T \cdot \mySigmaSub{\spa}{\pav} \| \leq \| \mySigmaSub{\pav}{\pav}^{-1}\| \|\myLambdaSub{\spa}{\pav}\| \| \mySigmaSub{\spa}{\pav} \|. 
	\]
	From \ref{mod:Lambda} we have $\|\myLambdaSub{\spa}{\pav}\| \leq \beta$ and from \ref{mod:offDiagonal} we have $ \| \mySigmaSub{\spa}{\pav} \| \leq \alpha \| \mySigmaSub{\pav}{\pav} \|$. Thus, we have 
	\begin{equation}
		\label{eq:denom1}
			\| \mySigmaSub{\pav}{\pav}^{-1}\| \|\myLambdaSub{\spa}{\pav}\| \| \mySigmaSub{\spa}{\pav} \| \leq \alpha \beta \kappa_0 < 1.
	\end{equation}
	The last inequality follows from the premise of the lemma. Therefore from Lemma~\ref{lem:normDifference} we have,
	\begin{align}
		\| \vec{A}_{v}^{-1} \| & \leq \| \mySigmaSub{\pav}{\pav}^{-1} \| \left( \frac{1}{1-\alpha \beta \kappa_0} \right). \label{eq:Ainverse-new}
	\end{align}
	
	Likewise, we can bound $\| \vec{b}_{v}\|$ as follows.
	\begin{align}
		\| \vec{b}_{v}\| 	& \leq \| \mySigmaSub{\pav}{v} - \myLambdaSub{\spa}{\pav}^T \cdot \mySigmaSub{\spa}{v} \|. \nonumber \\
		& \leq \| \mySigmaSub{\pav}{v} \| + \| \myLambdaSub{\spa}{\pav} \| \|\mySigmaSub{\spa}{v} \|. \nonumber \\
		& \leq \alpha (\| \mySigmaSub{\pav}{\pav} \| + \beta \| \mySigmaSub{\pav}{\pav} \|). \nonumber \\
		& \leq \alpha (1+\beta) \| \mySigmaSub{\pav}{\pav} \|. & \label{eq:B-new} 
	\end{align}
	Eq.~\eqref{eq:B-new} uses \ref{mod:offDiagonal} and \ref{mod:Lambda}.
	
	\underline{\textbf{Inductive step to bound $\| \vec{\Delta_{A_{v}}} \|$.}}
	\begin{align}
		\vec{\Delta_{A_{v}}} & = (\mySigmaTSub{\pav}{\pav} - \myLambdaTSub{\spa}{\pav}^T \cdot \mySigmaTSub{\spa}{\pav}) \nonumber \\ 
		& - (\mySigmaSub{\pav}{\pav} - \myLambdaSub{\spa}{\pav}^T \cdot \mySigmaSub{\spa}{\pav}) \nonumber
	\end{align}
	Taking spectral norm on both sides we get and using \ref{prop:1}
	\begin{align}
		& \| \vec{\Delta_{A_{v}}}\|  \leq \| \mySigmaTSub{\pav}{\pav} - \mySigmaSub{\pav}{\pav} \| \nonumber \\
		& + \| (\myLambdaTSub{\spa}{\pav} - \myLambdaSub{\spa}{\pav})^T \cdot \mySigmaSub{\spa}{\pav} \| + \| \myLambdaTSub{\spa}{\pav}^T \cdot \myEpsilonSub{\spa}{\pav} \| \label{eq:EqForDelta-new}
	\end{align}	
	
	We will now show that we have the following.
	\begin{equation}
		\label{eq:generalLambdaToSingle}
		\| \myLambdaTSub{\spa}{\pav}  - \myLambdaSub{\spa}{\pav} \| \leq k \eta \gamma.	
	\end{equation}
	From \ref{prop:11}, we have that for every vertex $w \in \pav$,
	\[
		\| \myLambdaTSub{\spa}{\pav}  - \myLambdaSub{\spa}{\pav} \| \leq k \| \myLambdaTSub{\spa}{w}  - \myLambdaSub{\spa}{w} \|
	\]
	Combining this with \ref{prop:11} and the inductive hypothesis for Eq.~\eqref{eq:LambdaInd} we get Eq.~\eqref{eq:generalLambdaToSingle}.
	
	Consider $\| \myLambdaTSub{\spa}{\pav} \|$. Using \ref{prop:1a} and Eq.~\eqref{eq:generalLambdaToSingle} we have
	\begin{align}
		\| \myLambdaTSub{\spa}{\pav} \| & \leq \| \myLambdaSub{\spa}{\pav} \| + k \eta \gamma	\nonumber \\
		& \leq \beta + k \eta \gamma \leq \beta  + \tau.	\label{eq:LambdaBound-new}
	\end{align}
	From Model~\ref{mod:mainModel} we have that $ \| \mySigmaTSub{\pav}{\pav} - \mySigmaSub{\pav}{\pav} \| =  \| \myEpsilonSub{\pav}{\pav}\|$.
	 
 	Therefore, substituting this and Eq.~\eqref{eq:LambdaBound-new} back into Eq.~\eqref{eq:EqForDelta-new} and using Eq.~\eqref{eq:generalLambdaToSingle} we get the following.
 	\begin{align}
 		\| \vec{\Delta_{A_{v}}}  \| & \leq  \| \myEpsilonSub{\pav}{\pav}\| + \| \myLambdaSub{\spa}{\pav} - \myLambdaTSub{\spa}{\pav} \| \|\mySigmaSub{\spa}{\pav}\| \nonumber \\
 		& + \|\myLambdaTSub{\spa}{\pav}\| \| \myEpsilonSub{\spa}{\pav}\| \nonumber \\
 		& \leq  \gamma \| \mySigmaSub{\pav}{\pav}\|  + k \eta \gamma \alpha \| \mySigmaSub{\pav}{\pav}\| + (\beta + \tau) \gamma \alpha \| \mySigmaSub{\pav}{\pav}\| \nonumber \\ 
 		& \leq \gamma  (k \eta + 1 + \beta + \tau) \|\mySigmaSub{\pav}{\pav}\|. \qquad \text{(Using $\alpha < 1$)}\label{eq:deltaA-new}
 	\end{align}
	
	\underline{\textbf{Inductive step to bound $\| \vec{\Delta_{b_{v}}} \|$.}}
	\begin{align*}
		\vec{\Delta_{b_{v}}} & = \left( \mySigmaTSub{\pav}{v} - \myLambdaTSub{\spa}{\pav}^T \cdot \mySigmaTSub{\spa}{v} \right) - \left( \mySigmaSub{\pav}{v} - \myLambdaSub{\spa}{\pav}^T \cdot \mySigmaSub{\spa}{v} \right) 
	\end{align*}
	Taking spectral norm and $2$-norm on both sides and using \ref{prop:1}, \ref{prop:2} we get,
	\begin{align}
		\|\vec{\Delta_{b_{v}}}\| & \leq \| \mySigmaTSub{\pav}{v} - \mySigmaSub{\pav}{v} \| + \| \myLambdaSub{\spa}{\pav} - \myLambdaTSub{\spa}{\pav} \| \| \mySigmaSub{\spa}{v} \| \nonumber \\
		& + \| \myLambdaTSub{\spa}{\pav} \| \| \myEpsilonSub{\spa}{v} \|. \nonumber \\
	& = \| \myEpsilonSub{\pav}{v} \| + \| \myLambdaSub{\spa}{\pav} - \myLambdaTSub{\spa}{\pav} \| \| \mySigmaSub{\spa}{v} \| \nonumber \\
	& + \| \myLambdaTSub{\spa}{\pav} \| \| \myEpsilonSub{\spa}{v} \|. \nonumber \\
	& \leq \alpha \gamma \| \mySigmaSub{\pav}{\pav} \| (1 + k \eta + \beta + \tau). \label{eq:lastEqBv-new}
	\end{align}
	The first summand in Eq.~\eqref{eq:lastEqBv-new} uses Claim~\ref{obs:ErrorVsSigma}. The second summand uses the Eq.~\eqref{eq:generalLambdaToSingle} and \ref{mod:offDiagonal}. The third summand uses Eq.~\eqref{eq:LambdaBound-new} and Claim~\ref{obs:ErrorVsSigma}.

	\underline{\textbf{Inductive step to bound $\| \tilde{\vec{A}}_{v}^{-1} \tilde{\vec{b}}_{v} - \vec{A}_{v}^{-1} \vec{b}_{v} \|$.}}\\
		Finally, we will now show the following.
		\[ \| \tilde{\vec{A}}_{v}^{-1} \tilde{\vec{b}}_{v} - \vec{A}_{v}^{-1} \vec{b}_{v} \| \leq \eta \gamma. \]
We can write $\tilde{\vec{A}}_{v} = \vec{A}_{v} + \vec{\Delta_{A_{v}}}$ and $\tilde{\vec{b}}_{v} = \vec{b}_{v} + \vec{\Delta_{b_{v}}}$. Thus we want to upper-bound,
\[
		\| \tilde{\vec{A}}_{v}^{-1} \cdot (\vec{b}_{v} + \vec{\Delta_{b_{v}}}) - \vec{A}_{v}^{-1} \cdot \vec{b}_{v} \|.
\]

Re-arranging we get the following.
\begin{align}
		\| \tilde{\vec{A}}_{v}^{-1} \cdot (\vec{b}_{v} + \vec{\Delta_{b_{v}}}) - \vec{A}_{v}^{-1} \cdot \vec{b}_{v} \| & = \| (\tilde{\vec{A}}_{v}^{-1} - \vec{A}_{v}^{-1}) \cdot \vec{b}_{v} + \tilde{\vec{A}}_{v}^{-1} \cdot \vec{\Delta_{b_{v}}} \| \nonumber \\
		 &=  \| (\tilde{\vec{A}}_{v}^{-1} - \vec{A}_{v}^{-1})  \cdot \vec{b}_{v}  +  \left( \vec{A}_{v} + \vec{\Delta_{A_{v}}} \right)^{-1} \cdot \vec{\Delta_{b_{v}}} \|. \nonumber \\
		 & \leq  \| \tilde{\vec{A}}_{v}^{-1} - \vec{A}_{v}^{-1} \|  \|  \vec{b}_{v} \| + \left \|  \left( \vec{A}_{v} + \vec{\Delta_{A_{v}}} \right)^{-1} \right \| \|  \vec{\Delta_{b_{v}}} \|.\label{eq:ToUseMatrixApprox-new}
\end{align}

Consider $ \| \tilde{\vec{A}}_{v}^{-1} - \vec{A}_{v}^{-1} \|  \|  \vec{b}_{v} \|$. First, we show that $\| \vec{A}_{v}^{-1} \| \| \vec{\Delta_{A_{v}}}  \|$ satisfies the premise in the matrix approximation Lemma~\ref{lem:matrixPerturbations}. From Eq.~\eqref{eq:Ainverse-new} we have $\| \vec{A}_{v}^{-1} \| \leq  \| \mySigmaSub{\pav}{\pav}^{-1} \|\tfrac{1}{1-\alpha \beta \kappa_0}$. From Eq.~\eqref{eq:deltaA-new} we have $\| \vec{\Delta_{A_{v}}}  \| \leq \gamma (k \eta + 1 + \beta + \tau) \|\mySigmaSub{\pav}{\pav}\|$. Thus, 
\[
	\| \vec{A}_{v}^{-1} \| \| \vec{\Delta_{A_{v}}}  \| \leq \tfrac{\gamma \kappa_0}{1-\alpha \beta \kappa_0} (k \eta + 1 + \beta + \tau).
\]
Using arguments similar to the one to obtain Eq.~\eqref{eq:preconditionReuse-new}, it can be shown that the RHS is at most $\frac{1}{2}$. Thus, it satisfies the premise in the matrix approximation Lemma~\ref{lem:matrixPerturbations}. Therefore, we have,
	\begin{multline}
		\label{eq:messyPart1-new}
		\| \tilde{\vec{A}}_{v}^{-1} - \vec{A}_{v}^{-1} \|  \|  \vec{b}_{v} \| 
		\leq \| \vec{A}_{v}^{-1} \| \| \vec{\Delta_{A_{v}}} \| \| \vec{A}_{v}^{-1} \| \| \vec{b}_{v} \| + 
		2 \| \vec{A}_{v}^{-1} \|^3 \| \vec{\Delta_{A_{v}}} \|^2 \| \vec{b}_{v} \|.
	\end{multline}
	Likewise, using Lemma~\ref{lem:matrixApprox} we get,
	\begin{multline}
		\label{eq:messyPart2-new}
		\left \|  \left( \vec{A}_{v} + \vec{\Delta_{A_{v}}} \right)^{-1} \right \| \|  \vec{\Delta_{b_{v}}} \| \leq \| \vec{A}_{v}^{-1}\| \| \vec{\Delta_{b_{v}}} \| + 2\| \vec{A}_{v}^{-1} \| \| \vec{\Delta_{A_{v}}} \| \| \vec{A}_{v}^{-1}\| \| \vec{\Delta_{b_{v}}} \|.
	\end{multline}
	Thus, combining Eq.~\eqref{eq:messyPart1-new} and Eq.~\eqref{eq:messyPart2-new}.
	\begin{align}
	& \| \tilde{\vec{A}}_{v}^{-1} - \vec{A}_{v}^{-1} \|  \|  \vec{b}_{v} \| + \left \|  \left( \vec{A}_{v} + \vec{\Delta_{A_{v}}} \right)^{-1} \right \| \|  \vec{\Delta_{b_{v}}} \| \nonumber \\
	& \leq \| \vec{A}_{v}^{-1} \|^2 \| \vec{\Delta_{A_{v}}} \| \| \vec{b}_{v} \| + 
		2 \| \vec{A}_{v}^{-1} \|^3 \| \vec{\Delta_{A_{v}}} \|^2 \| \vec{b}_{v} \| + \| \vec{A}_{v}^{-1} \| \|\vec{\Delta_{b_{v}}} \|  \nonumber \\
	&	+ 2 \| \vec{A}_{v}^{-1} \|^2 \| \vec{\Delta_{A_{v}}} \| \| \vec{\Delta_{b_{v}}} \| \label{eq:lastEqInductive-new} 
\end{align}
	The first summand in RHS Eq.~\eqref{eq:lastEqInductive-new} can be upper bounded by combining Eq.~\eqref{eq:Ainverse-new}, \eqref{eq:B-new} and \eqref{eq:deltaA-new} to obtain,
	\begin{align*}
		 & \| \vec{A}_{v}^{-1} \|^2  \| \vec{\Delta_{A_{v}}} \| \| \vec{b}_{v} \| \\
		 & \leq \left( \frac{1}{1-\alpha \beta \kappa_0} \| \mySigmaSub{\pav}{\pav}^{-1} \| \right)^2 \left( \gamma (k \eta + 1 + \beta + \tau) \| \mySigmaSub{\pav}{\pav} \| \right) \left( \alpha (1+\beta) \| \mySigmaSub{\pav}{\pav} \| \right). \\
		 & = \frac{\alpha \kappa_0^2  (1+\beta)(k \eta + 1 + \beta + \tau)}{(1- \alpha \beta \kappa_0)^2} \cdot \gamma.
	\end{align*}
	The third summand in Eq.~\eqref{eq:lastEqInductive-new} can be upper bounded by combining with Eq.~\eqref{eq:Ainverse-new} and \eqref{eq:lastEqBv-new} to obtain,
	\begin{align*}
		\| \vec{A}_{v}^{-1} \| \| \vec{\Delta_{b_{v}}} \| & \leq \left( \frac{1}{1-\alpha \beta \kappa_0} \| \mySigmaSub{\pav}{\pav}^{-1} \| \right) \left( \gamma \alpha (k \eta + 1 + \beta + \tau) \| \mySigmaSub{\pav}{\pav} \| \right)\\
		& \leq \frac{ \alpha \kappa_0  (k \eta + 1 + \beta + \tau)}{1-\alpha \beta \kappa_0} \cdot \gamma
	\end{align*}
	The second and fourth summands in Eq.~\eqref{eq:lastEqInductive-new} can be upper bounded by combining Eq.~\eqref{eq:Ainverse-new}, \eqref{eq:deltaA-new}, \eqref{eq:B-new} and \eqref{eq:lastEqBv-new} to obtain,
	\begin{align*}
		& 2  \| \vec{A}_{v}^{-1} \|^2  \|\vec{\Delta_{A_{v}}} \| \| \vec{\Delta_{b_{v}}} \| + 2  \| \vec{A}_{v}^{-1} \|^3  \|\vec{\Delta_{A_{v}}} \|^2 \| \vec{b}_{v} \| \\
		& \leq 2 \left( \frac{1}{1- \alpha \beta \kappa_0} \| \mySigmaSub{\pav}{\pav}^{-1} \| \right)^2 \left( \gamma (k \eta + 1 + \beta + \tau) \| \mySigmaSub{\pav}{\pav} \| \right) \\
		& \hspace{0.5 \linewidth} \left( \gamma \alpha (k \eta + 1 + \beta + \tau) \| \mySigmaSub{\pav}{\pav} \| \right) \\
		& + 2  \left( \frac{1}{1- \alpha \beta \kappa_0} \| \mySigmaSub{\pav}{\pav}^{-1} \| \right)^3 \left( \gamma (k \eta + 1 + \beta + \tau) \| \mySigmaSub{\pav}{\pav} \| \right)^2 \alpha (1+\beta) \| \mySigmaSub{\pav}{\pav} \|. \\
		& \leq c_6 \gamma^2 \qquad \text{where $c_6 = \frac{4 \alpha (1+ \beta) \kappa_0^3 (k \eta + 1 + \beta + \tau)^2}{(1-\alpha \beta \kappa_0)^3}$.}
	\end{align*}

	Thus, we get 
		$\| \tilde{\vec{A}}_{v}^{-1} \tilde{\vec{b}}_{v} - \vec{A}_{v}^{-1} \vec{b}_{v} \| \leq \left(\frac{\alpha \kappa_0^2  (1+\beta)(k \eta + 1 + \beta + \tau)}{(1-\alpha \beta \kappa_0)^2} + \frac{\kappa_0 \alpha (k \eta + 1 + \beta + \tau)}{1- \alpha \beta \kappa_0} \right) \gamma + c_6 \gamma^2$

	We now solve for the equation 
	\[
			\left(\frac{\alpha \kappa_0^2  (1+\beta)(k \eta + 1 + \beta + \tau)}{(1-\alpha \beta \kappa_0)^2} + \frac{\kappa_0  \alpha (k \eta + 1 + \beta + \tau)}{1-\alpha \beta \kappa_0} \right) + c_6 \gamma \leq \eta,
	\]
	which completes the induction.
	
	Thus, if we have 
	\begin{align}
		\label{eq:etaKConstant}
		\eta \left( 1 - \frac{k \alpha \kappa_0}{1-\alpha \beta \kappa_0} - \frac{k \alpha \kappa_0^2 (1+\beta)}{(1-\alpha \beta \kappa_0)^2} \right) & > \frac{\alpha \kappa_0^2 (1+\beta)(1+\beta+\tau)}{(1- \alpha \beta \kappa_0)^2} + \frac{\kappa_0 \alpha (1+\beta + \tau)}{1-\alpha \beta \kappa_0} + c_6 \gamma.
	\end{align}
	This implies we need,
	\[
			\eta > \left( \frac{\alpha \kappa_0^2 (1+\beta)(1+\beta+\tau)}{(1- \alpha \beta \kappa_0)^2} + \frac{\kappa_0 \alpha (1+\beta + \tau)}{1-\alpha \beta \kappa_0} \right) \left( 1 - \frac{k \alpha \kappa_0}{1-\alpha \beta \kappa_0} - \frac{k \alpha \kappa_0^2 (1+\beta)}{(1-\alpha \beta \kappa_0)^2} \right)^{-1} + o(1). 
	\]
	Moreover, the premise of the lemma ensures that $\frac{\alpha \kappa_0}{1-\alpha \beta \kappa_0} \left( 1 + \frac{\kappa_0 (1+\beta)}{1-\alpha \beta \kappa_0} \right) < \frac{0.99}{k} < 1$. Therefore, $\eta > 0$ and thus, we obtain Lemma~\ref{lem:MainInduction}. 
\end{proof}

\section{Random Model Parameters}
\label{sec:randomModel}
In this section, we will consider $\LSEM$s that are generated from random model parameters and show that they satisfy the model properties in Model~\ref{mod:mainModel}. Thus, we show that on a \emph{large} set of input parameters the assumptions in Model~\ref{mod:mainModel} hold with high-probability. Combining this with Theorem~\ref{thm:MainTheorem}  implies that inputs from this parameter space can be robustly identified using \emph{existing} algorithms \emph{provably}.

	\begin{model}[Generative model]
		\label{mod:generativeModel}
		Every non-zero entry in $\myLambda \in \mathbb{R}^{n \times n}$ is an i.i.d. sample from the uniform distribution $\cU\left[-\frac{1}{2k \mu}, \frac{1}{2k \mu}\right] \setminus \left[ -\frac{1}{n^2}, \frac{1}{n^2} \right]$ for some fixed $\mu \geq 10(k+1)$. The matrix $\myOmega \in \mathbb{R}^{n \times n}$ is generated as follows. We sample vectors $\vec{v}_1, \vec{v}_2, \ldots, \vec{v}_n \in \mathbb{R}^d$ from a $d$-dimensional unit sphere such that the following correlation holds. Each vector $\vec{v}_i$ is a uniform sample from the sub-space perpendicular to $\SPAN(\{ \vec{v}_j \}_{j \in V_{I-1}})$. The matrix $\myOmega$ is constructed by letting the $(i, j)$-th entry be $\langle \vec{v}_i, \vec{v}_j \rangle$. Thus, this matrix follows the zero-patterns mandated by the model.
	\end{model}
	
	For the Model~\ref{mod:generativeModel} defined above, we have the following theorem. 
	\begin{btheorem}
		\label{thm:random}	
		Let $\mu \geq 10 (k+1)$, $\alpha=\frac{1}{\mu} + o(1)$, $\beta=\frac{1}{\mu}$, $\lambda = n^2$, $\kappa_0 = \left( \frac{1+\mu}{\mu} \right)^4 + \frac{(\mu + 1)^2}{5\mu^2 (\mu-1)} + o(1)$. Then with probability at least $1-\frac{1}{\poly(n)}$ the following hold simultaneously.
		\begin{enumerate}
			\item For every $v \in V$ we have $\kappa(\mySigmaSub{\pav}{\pav}) \leq \kappa_0$.
			\item For every $v \in V$, we have that $\| \mySigmaSub{\pav}{v} \| \leq \alpha \| \mySigmaSub{\pav}{\pav} \|$, $\| \mySigmaSub{\spa}{\pav} \| \leq \alpha \| \mySigmaSub{\pav}{\pav} \|$ and $ \| \mySigmaSub{\spa}{v} \| \leq \alpha \| \mySigmaSub{\pav}{\pav} \|$.
			\item For every directed edge $(u \rightarrow v)$ in the causal DAG, we have that $\frac{1}{n^2} \leq | \Lambda_{u, v} |$. Moreover, for every $v \in V$ we have, $\| \myLambdaSub{\spa}{\pav} \| \leq \beta$.
		\end{enumerate}
	\end{btheorem}
	
	\xhdr{Proof Outline.} We prove high-probability bounds on the norm of sub-matrices of $\myOmega$ and $\myLambda$ using the concentration properties of the inner-product of the random vectors. We then use the Taylor series expansion for $(\Identity - \myLambda)^{-1}$ to obtain an expression for $\mySigma$. Using the various properties of the spectral norm of matrices, and the computed high-probability bounds we obtain the required bounds.
	\subsection{Proof of Theorem~\ref{thm:random}}
	We use the following addition notations in proof of Theorem~\ref{thm:random}. Define $F(d) := \exp[-\Omega(d^{0.5})]$. We define $\cI_d := \left[ -\frac{C_{conc}}{d^{0.25}}, \frac{C_{conc}}{d^{0.25}} \right]$ and $\cI_{d, 2} := \left[ -\frac{C_{conc}^2}{d^{0.5}}, \frac{C_{conc}^2}{d^{0.5}} \right]$ for the constant $C_{conc} =3$. We use $\mathcal{J}(k, n) := \frac{k^2. C_{conc}}{d^{0.25}}$. 
	
	We now give an expression for $\mySigma$ in terms of the matrices $\myLambda$ and $\myOmega$. From the Taylor series expansion (Section 3.3 in \cite{matrixBook}) we have,
		\begin{equation*}
			(\Identity - \myLambda)^{-1} = \Identity + \myLambda + \myLambda^2 + \ldots.
		\end{equation*}
		Note that $\myLambda^{n}_{i, j}$ denotes the set of all directed paths between $(i, j)$ of distance exactly $n$ (\cite{skiena1991implementing} pg. 230). Since there are $n$ vertices, the maximum length of any directed path is at most $n-1$. Thus, the matrices $\myLambda^n, \myLambda^{n+1}, \ldots$ are the all zeros vector. Therefore, the expression simplifies to,
		\begin{equation}
			\label{eq:taylor}
			(\Identity - \myLambda)^{-1} = \Identity + \myLambda + \myLambda^2 + \ldots + \myLambda^{n-1}.	
		\end{equation}
		
		From Eq~\eqref{eq:SigmaLSEM} we have,
		\begin{equation}
			\label{eq:SigmaLSEMT}
				\mySigma = \left(\myOmega^{1/2} \cdot (\Identity - \myLambda)^{-1} \right)^T \cdot \left( \myOmega^{1/2} \cdot (\Identity - \myLambda)^{-1} \right)
		\end{equation}
			
		Let $\pah{h}{v}$ denote the set of all parents of $v$ that are at a distance $h$ using the directed edges. Thus, $\pav = \pah{1}{v}$ and $v = \pah{0}{v}$. Combining Eq.~\eqref{eq:taylor} with Eq.~\eqref{eq:SigmaLSEM} we get,
		\begin{equation}
			\label{eq:SigmaParts2mid}
			\mySigmaSub{I}{J} = \sum_{i \geq 0} \sum_{j \geq 0} \left( \myLambdaSubSup{*}{I}{i} \right)^T \cdot \myOmega \cdot \myLambdaSubSup{*}{J}{j}
		\end{equation}
		However, note that the only non-zero entries in the matrix $\myLambda^i$ for the row indexed by $I$ is present in the columns $\pah{i}{I}$. Thus, one could further simplify Eq.~\eqref{eq:SigmaParts2mid} to obtain,
		\begin{equation}
			\label{eq:SigmaParts2}
			\mySigmaSub{I}{J} = \sum_{i \geq 0} \sum_{j \geq 0} \left( \myLambdaSubSup{\pah{i}{I}}{I}{i} \right)^T \cdot \myOmegaSub{\pah{i}{I}}{ \pah{j}{J}} \cdot \myLambdaSubSup{\pah{j}{J}}{J}{j}
		\end{equation}
	
	Given the generative process in Model~\ref{mod:generativeModel}, we obtain the following important properties of the matrices $\myLambda$ and $\myOmega$. We defer the proofs of these to the end of this section.
	
	\begin{lemma}[Corollary 4.4. in full version of \cite{ourUAI19}]
		\label{lem:Omega}	
		With probability at least $1-n^2F(d)$, $\mathbf{\Omega} \in \mathbb{R}^{n \times n}$ is a matrix with $\Omega_{i, i}=1$ for all $i \in [n]$ and $\Omega_{u, v} \in \mathcal{I}_d$ when $i \in V$ and $j \in V$ is such that the directed edge $i \not \in \pa{j}$ and $j \not \in \pa{i}$. 
	\end{lemma}
	\begin{lemma}
		\label{lem:OmegaRequired}
				For the randomly generated $\myOmega$ with probability at least $1-n^2 F(d)$ we have the following.
				\begin{enumerate}
					\item For every $I \subseteq[n]$ with $|I| \leq k^2$, $\| \myOmegaSub{I}{I} \| \leq 1+\mathcal{J}(k, n)$ and $\| \myOmegaSub{I}{I}^{-1} \| \leq 1+2 \mathcal{J}(k, n)$.
					\item When $|I| \leq k^i$ and $|J| \leq k^j$, we have that $\Norm{ \myOmegaSub{I}{J} }  \leq k^{\max\{i, j\}/2} \left(1 + \mathcal{J}(k, n) \right)$.
				\end{enumerate}
	\end{lemma}
	
	\begin{lemma}
		\label{lem:lambdaRand}
		For the generated $\myLambda$ we have that $\| \myLambda \| \leq \frac{1}{\mu}$ and for every $I, J \subseteq [n]$ we have $\| \myLambdaSub{I}{J} \| \leq \frac{1}{\mu}$. 
	\end{lemma}
	\begin{claim}
		\label{lem:dBound}	
		Let $d \geq \dbound$. Then $\mathcal{J}(k, n) \leq \frac{1}{\log (n)} = o(1)$.
	\end{claim}

	We now prove Theorem~\ref{thm:random}. In particular, we have the following.
	\begin{enumerate}
		\item Lemma~\ref{lem:conditionNumberSigma} proves that for every $v \in V$ we have $\kappa(\mySigmaSub{\pav}{\pav}) \leq \left( \frac{1+\mu}{\mu} \right)^4 + \frac{(\mu + 1)^2}{5\mu^2 (\mu-1)} + o(1)$ with probability at least $1-\frac{1}{\poly(n)}$. Thus this proves \ref{mod:condNumberInput}.
		\item Lemma~\ref{lem:SigmaUpperBoundWithDiagonal} proves that for every $v \in V$, we have that $\| \mySigmaSub{\pav}{v} \| \leq \alpha \| \mySigmaSub{\pav}{\pav} \|$, $\| \mySigmaSub{\spa}{\pav} \| \leq \alpha \| \mySigmaSub{\pav}{\pav} \|$ and $ \| \mySigmaSub{\spa}{v} \| \leq \alpha \| \mySigmaSub{\pav}{\pav} \|$ with probability at least $1-\frac{1}{\poly(n)}$ for $\alpha=\frac{1}{\mu} + o(1)$. Thus this proves \ref{mod:offDiagonal}.
		\item In Lemma~\ref{lem:lambdaRand} plugging in $I=\spa$ and $J = \pav$ we get that $\| \myLambdaSub{\spa}{\pav} \| \leq \frac{1}{\mu}$ holds. Thus this proves \ref{mod:Lambda}. 
		\item To note that the premise in Theorem~\ref{thm:MainTheorem} holds, first note that $\alpha \beta \kappa_0 < 1 $ when $\mu \geq 10 (k+1)$. Likewise we can verify that,
		\[
				\frac{\alpha \kappa_0}{1-\alpha \beta \kappa_0} \left( 1 + \frac{\kappa_0 (1+\beta)}{1-\alpha \beta \kappa_0} \right) < 0.99.
		\]
		for the constants $\alpha = \frac{1}{\mu} + o(1)$, $\beta=\frac{1}{\mu}$, $\lambda = n^2$, $\kappa_0 = \left( \frac{1+\mu}{\mu} \right)^4 + \frac{(\mu + 1)^2}{5\mu^2 (\mu-1)} + o(1)$ and $\mu \geq 10 (k+1)$. 
	\end{enumerate}	
	\begin{lemma}
		\label{lem:conditionNumberSigma}
		Consider the generative process in Model~\ref{mod:generativeModel} and consider data generated from this model. Let $\mySigma \in \mathbb{R}^{n \times n}$ denote the corresponding data covariance matrix. Then we have that with probability at least $1-\frac{1}{\poly(n)}$ for every $v \in V$, the condition number $\kappa(\mySigmaSub{\pav}{\pav}) \leq \kappa_0$ where $\kappa_0 = \left( \frac{1+\mu}{\mu} \right)^4 + \frac{(\mu + 1)^2}{5\mu^2 (\mu-1)} + o(1)$.
	\end{lemma}
	\begin{proof}
		We prove this statement by proving the following equations each hold with probability at least $1-\frac{1}{\poly(n)}$. Thus, taking a union bound and combining them gives the statement of the lemma.
		 \begin{equation}
		 	\label{eq:condNumbSigma}
		 		\| \mySigmaSub{\pav}{\pav} \| \leq \left( \frac{1+\mu}{\mu} \right)^2 + \frac{1}{5(\mu-1)} + o(1).
		 \end{equation}
		 \begin{equation}
		 	\label{eq:condNumbSigmaInv}	
		 	\| \mySigmaSub{\pav}{\pav}^{-1} \| \leq \left( \frac{1+\mu}{\mu} \right)^2 + o(1).
		 \end{equation}

		\underline{\textbf{Proof of Eq.~\eqref{eq:condNumbSigma}}}
		
		Consider $\mySigmaSub{\pav}{\pav}$.  Let $I = \pav$ and $J=\pav$ in Eq.~\eqref{eq:SigmaParts2}. Taking the norm on both sides and using \ref{prop:1} we get,
		\begin{multline}
			\label{eq:singmaPavPavNormUB}
			\| \mySigmaSub{\pav}{\pav} \| \leq \sum_{i \geq 0} \sum_{j \geq 0}  \| \myLambdaSubSup{\pah{i}{\pav}}{\pav}{i} \| \| \myOmegaSub{\pah{i}{\pav}}{ \pah{j}{\pav}} \| \| \myLambdaSubSup{\pah{j}{\pav}}{\pav}{j}\|.
		\end{multline}
		
		Recall that $\myLambdaSubSup{\pah{i}{\pav}}{\pav}{i}$ denotes the sub-matrix indexed by $(\pah{i}{\pav}, \pav)$ of the matrix $\myLambda^i$. From \ref{prop:1}, \ref{prop:3}, \ref{prop:11} and Lemma~\ref{lem:lambdaRand} we have,

		\begin{equation}
			\label{eq:LambdaiNorm}
			\| \myLambdaSubSup{\pah{i}{\pav}}{\pav}{i} \| \leq \| \myLambda \|^i \leq \frac{1}{\mu^i}
		\end{equation}
		and
		\begin{equation}
			\label{eq:LambdajNorm}
			\| \myLambdaSubSup{\pah{j}{\pav}}{\pav}{j} \| \leq \| \myLambda \|^j \leq \frac{1}{\mu^j}
		\end{equation}
		From Lemma~\ref{lem:OmegaRequired}, we have that 
		\begin{equation}
			\label{eq:OmegapavNorm}
			\| \myOmegaSub{\pah{i}{\pav}}{ \pah{j}{\pav}} \| \leq k^{\max\{i, j\}/2} \left(1 + \frac{\mathcal{J}(k, n)}{k} \right).
		\end{equation} 
		
		We can rewrite Eq.~\eqref{eq:singmaPavPavNormUB} as,
		\begin{multline}
			\label{eq:singmaPavPavNormUBreWrite}	
			\| \mySigmaSub{\pav}{\pav} \| \leq	\sum_{0 \leq i \leq 1} \sum_{0 \leq j \leq 1} \myLambdaSubSup{\pah{i}{\pav}}{\pav}{i} \| \| \myOmegaSub{\pah{i}{\pav}}{ \pah{j}{\pav}} \| \| \myLambdaSubSup{\pah{j}{\pav}}{\pav}{j}\| + \\
			2 \ast \sum_{i \geq 2} \sum_{j \geq i} \myLambdaSubSup{\pah{i}{\pav}}{\pav}{i} \| \| \myOmegaSub{\pah{i}{\pav}}{ \pah{j}{\pav}} \| \| \myLambdaSubSup{\pah{j}{\pav}}{\pav}{j}\|.
		\end{multline}
		
		Plugging in Eq.~\eqref{eq:LambdaiNorm}, Eq.~\eqref{eq:LambdajNorm} and Eq.~\eqref{eq:OmegapavNorm} into Eq.~\eqref{eq:singmaPavPavNormUBreWrite} and using Lemma~\ref{lem:OmegaRequired}, RHS can be upper-bounded by,
		\begin{multline}
			\label{eq:singmaPavPavNormUBPlugIn}	
			\leq 1 + \frac{2}{\mu} \left( 1 + \mathcal{J}(k, n) \right)	+ \frac{1 + \mathcal{J}(k, n)}{\mu^2} +
			2 \ast \sum_{i \geq 1}\sum_{j \geq i} \frac{k^{j/2}}{\mu^{i+j}} \left(1 + \frac{\mathcal{J}(k, n)}{k} \right).
		\end{multline}
		Consider the last summation. Note that the choice of $\mu \geq 10 (k + 1)$ implies that $k \leq \frac{\mu}{10} - 1 \leq \frac{\mu}{10}$. Using the fact that $\sum_{j \geq 1} \frac{k^{j/2}}{\mu^j} \leq \frac{1}{10}$, we can upper-bound the last summation by, 
		\begin{equation}
			\label{eq:geometricSeries}
			\leq 2 \ast \sum_{i \geq 1}\frac{1}{\mu^i} \left(1 + \frac{\mathcal{J}(k, n)}{k} \right) \leq \frac{2}{10 (\mu-1)} \left(1 + \frac{\mathcal{J}(k, n)}{k} \right).
		\end{equation}
		The last inequality uses the geometric series sum. Plugging in Eq.~\eqref{eq:geometricSeries} into Eq.~\eqref{eq:singmaPavPavNormUBPlugIn}, we get Eq.~\eqref{eq:condNumbSigma}.
	
	\underline{\textbf{Proof of Eq.~\eqref{eq:condNumbSigmaInv}}}
	
	Consider $\mySigmaSub{\pav}{\pav}^{-1}$. From Eq.~\eqref{eq:SigmaLSEM} we have,
	\begin{equation}
		\label{eq:SigmaInvBasic}
		\mySigmaSub{\pav}{\pav}^{-1} = \left( (\Identity - \myLambda) \cdot \myOmega^{-1} \cdot (\Identity - \myLambda)^T \right)_{[\pav, \pav]}.
	\end{equation}
	Note that the set of non-zero elements in $(\Identity - \myLambda)_{[\pav, *]}$ is essentially on columns indexed by $\ch{\pav}$. Taking the spectral norm and using \ref{prop:1} and \ref{prop:2} we get,
	\begin{equation}
		\label{eq:SigmaInvBasic2}
		\Norm{\mySigmaSub{\pav}{\pav}^{-1} } \leq \Norm{ \myOmegaSub{\ch{\pav}}{\ch{\pav}}^{-1} } \Norm{ (\Identity - \myLambda) }^2.
	\end{equation}
	From Lemma~\ref{lem:OmegaRequired} we can upper-bound the RHS of Eq.~\eqref{eq:SigmaInvBasic2} by,
	\begin{equation}
		\label{eq:SigmaInvBasic3}
		\Norm{ \myOmegaSub{\ch{\pav}}{\ch{\pav}}^{-1} } \Norm{ (\Identity - \myLambda) }^2 \leq (1 + 2 \cJ(k, n)) \Norm{ (\Identity - \myLambda) }^2.
	\end{equation}
	Combining Lemma~\ref{lem:lambdaRand} with \ref{prop:1} the RHS in Eq.~\eqref{eq:SigmaInvBasic3} can be upper-bounded by,
	\begin{equation*}
		(1 + 2 \cJ(k, n)) \Norm{ (\Identity - \myLambda) }^2 \leq (1 + 2 \cJ(k, n)) \left( 1 + \frac{1}{\mu} \right)^2.
	\end{equation*}
	Combining this with Claim~\ref{lem:dBound} we get,
	\begin{equation*}
		\Norm{\mySigmaSub{\pav}{\pav}^{-1} } \leq \left( \frac{1+\mu}{\mu} \right)^2 + o(1). \qedhere
	\end{equation*}
	\end{proof}

	\begin{lemma}
		\label{lem:SigmaUpperBoundWithDiagonal}
		Consider the random generation process described in Model~\ref{mod:generativeModel}. Then the $\mySigma$ corresponding to this process satisfies the following with probability at least $1-\frac{1}{\poly(n)}$ for $\alpha=\frac{1}{\mu} + o(1)$ and $\forall v \in V$. 
		\begin{align*}
			& \| \mySigmaSub{\pav}{v} \| \leq \alpha \| \mySigmaSub{\pav}{\pav} \| \\
			&  \| \mySigmaSub{\spa}{\pav} \| \leq \alpha \| \mySigmaSub{\pav}{\pav} \| \\
			&  \| \mySigmaSub{\spa}{v} \| \leq \alpha \| \mySigmaSub{\pav}{\pav} \|
		\end{align*}
	\end{lemma}
	
	\begin{proof}
		We show that the following holds with probability at least $1-\frac{1}{\poly(n)}$. 
		\begin{align*}
			& \| \mySigmaSub{\pav}{v} \| - \alpha \| \mySigmaSub{\pav}{\pav} \| \leq 0.	 \\
			& \| \mySigmaSub{\spa}{\pav} \| - \alpha \| \mySigmaSub{\pav}{\pav} \| \leq 0.	\\
			& \| \mySigmaSub{\spa}{v} \| - \alpha \| \mySigmaSub{\pav}{\pav} \| \leq 0.	
		\end{align*}
		The Lemma follows from these inequalities. 
		
		In the proof of this inequality, we use the following fact. This follows from the fact that a path of length $j$ from $x$ to $y$ can be decomposed into a path of length $j-1$ from $x$ to $z \in \pa{y}$ and an edge from $z$ to $y$.
		\begin{fact}
			\label{fact:pathi}
			For any $j \geq 2$ we have that $\myLambdaSubSup{\pah{j}{v}}{v}{j} = \myLambdaSubSup{\pah{j-1}{\pav}}{\pav}{j-1} \cdot \myLambdaSub{\pav}{v}$. Likewise for any $j \geq 3$, we have $\myLambdaSubSup{\pah{j}{v}}{v}{j} = \myLambdaSubSup{\pah{j-2}{\spa}}{\spa}{j-2} \cdot \myLambdaSubSup{\spa}{v}{2}$.
		\end{fact}
		
		Consider $\mySigmaSub{\pav}{v}$. Using $I= \pav$ and $J=\{v\}$ in Eq.~\eqref{eq:SigmaParts2} this can be written as,
		\begin{align*}
			& \mySigmaSub{\pav}{v} \\
			& = \sum_{i \geq 0} \sum_{j \geq 0} \left( \myLambdaSubSup{\pah{i}{\pav}}{\pav}{i} \right)^T \cdot \myOmegaSub{\pah{i}{\pav}}{ \pah{j}{v}} \cdot \myLambdaSubSup{\pah{j}{v}}{v}{j} \\
			& = \sum_{i \geq 0} \sum_{j \geq 1} \left( \myLambdaSubSup{\pah{i}{\pav}}{\pav}{i} \right)^T \cdot \myOmegaSub{\pah{i}{\pav}}{ \pah{j}{v}} \cdot \myLambdaSubSup{\pah{j}{v}}{v}{j} \\
			& \qquad + \sum_{i \geq 0} \left( \myLambdaSubSup{\pah{i}{\pav}}{\pav}{i} \right)^T \cdot \myOmegaSub{\pah{i}{\pav}}{v}  \\
			& = \sum_{i \geq 0} \sum_{j \geq 0} \left( \myLambdaSubSup{\pah{i}{\pav}}{\pav}{i} \right)^T \cdot \myOmegaSub{\pah{i}{\pav}}{ \pah{j}{\pav}} \cdot \myLambdaSubSup{\pah{j}{\pav}}{\pav}{j} \cdot \myLambdaSub{\pav}{v} \\
			& \qquad + \sum_{i \geq 0} \left( \myLambdaSubSup{\pah{i}{\pav}}{\pav}{i} \right)^T \cdot \myOmegaSub{\pah{i}{\pav}}{v}  \\
			& = \mySigmaSub{\pav}{\pav} \cdot \myLambdaSub{\pav}{v}  \\
			& \qquad + \sum_{i \geq 0} \left( \myLambdaSubSup{\pah{i}{\pav}}{\pav}{i} \right)^T \cdot \myOmegaSub{\pah{i}{\pav}}{v}  \\
		\end{align*}
		The third equality uses Fact~\ref{fact:pathi} and reindexing $j$. Thus, taking the norm on both sides and using the norm properties in Lemma~\ref{lem:normProp}, we get,
		\begin{equation*}
			\| \mySigmaSub{\pav}{v} \| \leq 	\| \myLambdaSub{\pav}{v} \| \| \mySigmaSub{\pav}{\pav} \| + \sum_{i \geq 0} \Norm{ \myLambdaSubSup{\pah{i}{\pav}}{\pav}{i} } \Norm{\myOmegaSub{\pah{i}{\pav}}{v} }.
		\end{equation*}
		Consider $\sum_{i \geq 0} \Norm{ \myLambdaSubSup{\pah{i}{\pav}}{\pav}{i} } \Norm{\myOmegaSub{\pah{i}{\pav}}{v} }$. Using $I=\pah{i}{\pav}$ and $J=\{v\}$ in Lemma~\ref{lem:OmegaRequired} the quantity $\| \myOmegaSub{\pah{i}{\pav}}{v} \|$ can be upper-bounded by $\mathcal{J}(k, n)$. Using $I=\pah{i}{\pav}$ and $J=\pav$ in Lemma~\ref{lem:lambdaRand} and \ref{prop:2} we have that $\Norm{ \myLambdaSubSup{\pah{i}{\pav}}{\pav}{i} } \leq \frac{1}{\mu^i}$. Thus we have,
		 \begin{equation}
		 	\label{eq:randomAssumpMain}	
		 	\sum_{i \geq 0} \Norm{ \myLambdaSubSup{\pah{i}{\pav}}{\pav}{i} } \Norm{\myOmegaSub{\pah{i}{\pav}}{v} } \leq \mathcal{J}(k, n) \sum_{i \geq 0} \frac{1}{\mu^i} \leq \mathcal{J}(k, n) \left( \frac{\mu}{\mu-1} \right).
		 \end{equation}
		Thus, we have
		\[
				\| \mySigmaSub{\pav}{v} \| - \alpha \| \mySigmaSub{\pav}{\pav} \| \leq (\| \myLambdaSub{\pav}{v} \| - \alpha) \| \mySigmaSub{\pav}{\pav} \| + \mathcal{J}(k, n) \left( \frac{\mu}{\mu-1} \right).
		\]
		We want $(\| \myLambdaSub{\pav}{v} \| - \alpha) \| \mySigmaSub{\pav}{\pav} \| + \mathcal{J}(k, n) \cdot \frac{\mu}{\mu-1} \leq 0$. Re-arranging, we get
		\begin{equation}
			\label{eq:alphaFirst}	
					\alpha \geq \| \myLambdaSub{\pav}{v} \| + \frac{1}{\| \mySigmaSub{\pav}{\pav} \|} \left( \frac{\mu}{\mu-1} \cdot \mathcal{J}(k, n) \right).
		\end{equation}

		Likewise, consider $\| \mySigmaSub{\spa}{\pav} \|$. Using $I= \spa$ and $J=\pav$ in Eq.~\eqref{eq:SigmaParts2} this can be written as,
		\begin{align*}
			& \mySigmaSub{\spa}{\pav} \\
			& = \sum_{i \geq 0} \sum_{j \geq 0} \left( \myLambdaSubSup{\pah{i}{\spa}}{\spa}{i} \right)^T \cdot \myOmegaSub{\pah{i}{\spa}}{ \pah{j}{\pav}} \cdot \myLambdaSubSup{\pah{j}{\pav}}{\pav}{j} \\
			& = \sum_{i \geq 0} \sum_{j \geq 1} \left( \myLambdaSubSup{\pah{i}{\spa}}{\spa}{i} \right)^T \cdot \myOmegaSub{\pah{i}{\spa}}{\pah{j}{\pav}} \cdot \myLambdaSubSup{\pah{j}{\pav}}{\pav}{j} \\
			& \qquad + \sum_{i \geq 0} \left( \myLambdaSubSup{\pah{i}{\spa}}{\spa}{i} \right)^T \cdot \myOmegaSub{\pah{i}{\spa}}{\pav}  \\
			& = \sum_{i \geq 0} \sum_{j \geq 0} \left( \myLambdaSubSup{\pah{i}{\spa}}{\spa}{i} \right)^T \cdot \myOmegaSub{\pah{i}{\spa}}{ \pah{j}{\spa}} \cdot \myLambdaSubSup{\pah{j}{\spa}}{\spa}{j} \cdot \myLambdaSub{\spa}{\pav} \\
			& \qquad + \sum_{i \geq 0} \left( \myLambdaSubSup{\pah{i}{\spa}}{\spa}{i} \right)^T \cdot \myOmegaSub{\pah{i}{\spa}}{\pav}  \\
			& = \mySigmaSub{\spa}{\spa} \cdot \myLambdaSub{\spa}{\pav}  \\
			& \qquad + \sum_{i \geq 0} \left( \myLambdaSubSup{\pah{i}{\spa}}{\spa}{i} \right)^T \cdot \myOmegaSub{\pah{i}{\spa}}{\pav}  \\
		\end{align*}
		The third equality uses Fact~\ref{fact:pathi} and reindexing $j$. Thus, taking the norm on both sides and using the norm properties in Lemma~\ref{lem:normProp}, we get,
		\begin{multline*}
			\Norm{\mySigmaSub{\spa}{\pav}} \leq 	\Norm{ \myLambdaSub{\spa}{\pav} } \Norm{ \mySigmaSub{\spa}{\spa} } + \\ \sum_{i \geq 0} \Norm{ \myLambdaSubSup{\pah{i}{\spa}}{\spa}{i} } \Norm{ \myOmegaSub{\pah{i}{\spa}}{\pav} }.
		\end{multline*}
		As in Eq.~\eqref{eq:randomAssumpMain}, we now upper-bound the second summand. Using $I=\pah{i}{\spa}$ and $J=\pav$ in Lemma~\ref{lem:OmegaRequired} the quantity $\| \myOmegaSub{\pah{i}{\spa}}{\pav} \|$ can be upper-bounded by $\mathcal{J}(k, n)$. Using $I=\pah{i}{\spa}$ and $J=\spa$ in Lemma~\ref{lem:lambdaRand} and \ref{prop:2} we have that $\Norm{ \myLambdaSubSup{\pah{i}{\spa}}{\spa}{i} } \leq \frac{1}{\mu^i}$. Thus we have,
		\[
				\textstyle\sum_{i \geq 0} \Norm{ \myLambdaSubSup{\pah{i}{\spa}}{\spa}{i} } \Norm{ \myOmegaSub{\pah{i}{\spa}}{\pav} } \leq \mathcal{J}(k, n) \sum_{i \geq 0} \frac{1}{\mu^i} \leq \mathcal{J}(k, n) \cdot \frac{\mu}{\mu-1}.
		\]
		Moreover from Eq.~\eqref{eq:SigmaParts2} we have,
		\begin{multline*}
			\mySigmaSub{\spa}{\spa} = \mySigmaSub{\pav}{\pav} - \sum_{i \geq 0} \myLambdaSubSup{\pah{i}{\pav}}{\pav}{i} \cdot \myOmegaSub{\pah{i}{\pav}}{\pav} \\ - \sum_{j \geq 0} \myOmegaSub{\pah{j}{\pav}}{\pav} \cdot \myLambdaSubSup{\pah{j}{\pav}}{\pav}{j}. 
		\end{multline*}
		Taking spectral norm on both sides and using \ref{prop:1} and \ref{prop:2} we can upper-bound it by,
		\begin{align}
			& \| \mySigmaSub{\spa}{\spa} \| \nonumber \\
			&  \leq \Norm{ \mySigmaSub{\pav}{\pav}} + \sum_{i \geq 0} \Norm{ \myLambdaSubSup{\pah{i}{\pav}}{\pav}{i} } \Norm{ \myOmegaSub{\pah{i}{\pav}}{\pav} } \nonumber \\
			& + \sum_{j \geq 0} \Norm{ \myOmegaSub{\pah{j}{\pav}}{\pav} } \Norm{ \myLambdaSubSup{\pah{j}{\pav}}{\pav}{j} }. \label{eq:sigmai1i1}
		\end{align}
		First note that the last two summands are the same. Second, as in Eq.~\eqref{eq:randomAssumpMain} we can upper-bound it by,
		\begin{equation}
			\label{eq:partthree}	
			\textstyle \sum_{j\geq0} \Norm{ \myOmegaSub{\pah{j}{\pav}}{\pav} } \Norm{ \myLambdaSubSup{\pah{j}{\pav}}{\pav}{j} } \leq \mathcal{J}(k, n) \sum_{j \geq 0} \frac{1}{\mu^j} \leq \mathcal{J}(k, n) \cdot \frac{\mu}{\mu-1}.
		\end{equation}
		Thus we have,
		\begin{equation}
			\label{eq:SigmaTwoBeforeSigma}
			\| \mySigmaSub{\spa}{\spa} \| \leq 	 \Norm{ \mySigmaSub{\pav}{\pav}} + \mathcal{J}(k, n) \cdot \frac{\mu}{\mu-1}.
		\end{equation}
		Therefore we have,
		\[
			\textstyle \| \mySigmaSub{\spa}{\pav} \| - \alpha \| \mySigmaSub{\pav}{\pav} \| \leq (\| \myLambdaSub{\spa}{\pav} \| - \alpha) \| \mySigmaSub{\pav}{\pav} \| + \mathcal{J}(k, n) \cdot \frac{3 \mu}{\mu-1}.
		\]
		We want the RHS to be upper-bounded by $0$. Thus, re-arranging we have,
		\begin{equation}
			\label{eq:alphaSecond}	
			\textstyle	\alpha \geq \| \myLambdaSub{\spa}{\pav} \| + \frac{1}{\| \mySigmaSub{\pav}{\pav} \|} \left( \frac{3 \mu}{\mu-1} \cdot \mathcal{J}(k, n) \right).
		\end{equation}
		Finally, consider $\| \mySigmaSub{\spa}{v} \|$. Using $I= \spa$ and $J=\{ v \}$ in Eq.~\eqref{eq:SigmaParts2} this can be written as,
		\begin{align*}
			& \mySigmaSub{\spa}{v} \\
			& = \sum_{i \geq 0} \sum_{j \geq 0} \left( \myLambdaSubSup{\pah{i}{\spa}}{\spa}{i} \right)^T \cdot \myOmegaSub{\pah{i}{\spa}}{ \pah{j}{v}} \cdot \myLambdaSubSup{\pah{j}{v}}{v}{j} \\
			& = \sum_{i \geq 0} \sum_{j \geq 2} \left( \myLambdaSubSup{\pah{i}{\spa}}{\spa}{i} \right)^T \cdot \myOmegaSub{\pah{i}{\spa}}{ \pah{j}{v}} \cdot \myLambdaSubSup{\pah{j}{v}}{v}{j} \\
			&  \hspace{0.3 \linewidth} + \sum_{i \geq 0} \left( \myLambdaSubSup{\pah{i}{\spa}}{\spa}{i} \right)^T \cdot \myOmegaSub{\pah{i}{\spa}}{v} \\
			& \hspace{0.35 \linewidth}  + \sum_{i \geq 0} \left( \myLambdaSubSup{\pah{i}{\spa}}{\spa}{i} \right)^T \cdot \myOmegaSub{\pah{i}{\spa}}{ \pav } \cdot \myLambdaSub{\pav}{v} \\ 
			& = \sum_{i \geq 0} \sum_{j \geq 0} \left( \myLambdaSubSup{\pah{i}{\spa}}{\spa}{i} \right)^T \cdot \myOmegaSub{\pah{i}{\spa}}{ \pah{j}{\spa}} \cdot \myLambdaSubSup{\pah{j}{\spa}}{\spa}{j} \cdot \myLambdaSubSup{\spa}{v}{2} \\
			& \hspace{0.3 \linewidth} + \sum_{i \geq 0} \left( \myLambdaSubSup{\pah{i}{\spa}}{\spa}{i} \right)^T \cdot \myOmegaSub{\pah{i}{\spa}}{v} \\
			& \hspace{0.35 \linewidth} + \sum_{i \geq 0} \left( \myLambdaSubSup{\pah{i}{\spa}}{\spa}{i} \right)^T \cdot \myOmegaSub{\pah{i}{\spa}}{ \pav } \cdot \myLambdaSub{\pav}{v} \\
			& = \mySigmaSub{\spa}{\spa} \cdot \myLambdaSubSup{\spa}{v}{2} \\
			& \hspace{0.3 \linewidth} + \sum_{i \geq 0} \left( \myLambdaSubSup{\pah{i}{\spa}}{\spa}{i} \right)^T \cdot \myOmegaSub{\pah{i}{\spa}}{v} \\
			& \hspace{0.35 \linewidth} + \sum_{i \geq 0} \left( \myLambdaSubSup{\pah{i}{\spa}}{\spa}{i} \right)^T \cdot \myOmegaSub{\pah{i}{\spa}}{ \pav } \cdot \myLambdaSub{\pav}{v} \\
		\end{align*}
		The third equality uses Fact~\ref{fact:pathi} and reindexing $j$. Taking spectral norm on both sides and using \ref{prop:1} and \ref{prop:2} we get,
		\begin{align*}
			& \Norm{ \mySigmaSub{\spa}{v} }  \\
			& \leq \Norm{ \mySigmaSub{\spa}{\spa} } \Norm{\myLambdaSubSup{\spa}{v}{2}} + \sum_{i \geq 0} \Norm{ \myLambdaSubSup{\pah{i}{\spa}}{\spa}{i} } \Norm{ \myOmegaSub{\pah{i}{\spa}}{v} } \\
			& + \sum_{i \geq 0} \Norm{\myLambdaSubSup{\pah{i}{\spa}}{\spa}{i}} \Norm{ \myOmegaSub{\pah{i}{\spa}}{ \pav } } \Norm{ \myLambdaSub{\pav}{v} }.
		\end{align*}
		Using arguments similar to Eq.~\eqref{eq:randomAssumpMain} we have,
		\begin{align}
			& \textstyle \sum_{i \geq 0} \Norm{ \myLambdaSubSup{\pah{i}{\spa}}{\spa}{i} } \Norm{ \myOmegaSub{\pah{i}{\spa}}{v} }  \leq \frac{1+\mu}{\mu} \cdot \mathcal{J}(k, n) \label{eq:one} \\
			& \textstyle \sum_{i \geq 0} \Norm{\myLambdaSubSup{\pah{i}{\spa}}{\spa}{i}} \Norm{ \myOmegaSub{\pah{i}{\spa}}{ \pav } } \Norm{ \myLambdaSub{\pav}{v} }  \leq \frac{\mu}{\mu-1} \cdot \mathcal{J}(k, n) \label{eq:two} 
		\end{align}
		Thus, combining Eq.~\eqref{eq:one}, \eqref{eq:two}, \eqref{eq:partthree}, \eqref{eq:SigmaTwoBeforeSigma} we get,
		\[
			\textstyle \| \mySigmaSub{\spa}{v} \| - \alpha \| \mySigmaSub{\pav}{\pav} \| \leq \left( \Norm{\myLambdaSubSup{\spa}{v}{2}} - \alpha \right) \| \mySigmaSub{\pav}{\pav} \| + \mathcal{O}(\mathcal{J}(k, n)).
		\]
		We want RHS to be upper-bounded by $0$. Thus, re-arranging we have,
		\begin{equation}
			\label{eq:alphaThird}
			\textstyle	\alpha \geq \Norm{\myLambdaSubSup{\spa}{v}{2}} + \frac{1}{\| \mySigmaSub{\pav}{\pav} \|} \mathcal{O}(\mathcal{J}(k, n)).
		\end{equation}
		As proved in Eq.~\eqref{eq:condNumbSigma} we have that $\| \mySigmaSub{\pav}{\pav} \| \leq \left(\frac{\mu}{1 + \mu} \right)^2 + \frac{1}{5(\mu-1)} + \mathcal{O}(\mathcal{J}(k, n))$ with probability at least $1-\frac{1}{\poly(n)}$. Thus combining it with Eq.~\eqref{eq:alphaFirst}, \eqref{eq:alphaSecond} and \eqref{eq:alphaThird} we have,
		 \begin{equation}
		 	\label{eq:Sigmaii1final}	
		 	\alpha \geq \max \left\{ \Norm{\myLambdaSub{\pav}{v}}, \Norm{\myLambdaSubSup{\spa}{v}{2}} \right\} + \mathcal{O}\left( \mathcal{J}(k, n) + \mathcal{J}(k, n)^2 \right).
		 \end{equation}
		 From Lemma~\ref{lem:lambdaRand} we have $\Norm{\myLambdaSub{\pav}{v}} \leq \frac{1}{\mu}$ and $\Norm{\myLambdaSubSup{\spa}{v}{2}} \leq \frac{1}{\mu^2} \leq \frac{1}{\mu}$. Since the RHS in Eq.~\eqref{eq:Sigmaii1final} is an increasing function in these quantities it suffices if we have $\alpha = \frac{1}{\mu} + o(1)$.
	\end{proof}
	

		\subsection{Proof of Lemma~\ref{lem:OmegaRequired}}
		\begin{proof}
		Using Lemma~\ref{lem:Omega} we have that with probability at least $1- F(d)$, the diagonal elements of $\mathbf{\Omega}$ are $1$ and the non-diagonal elements are in the interval $\mathcal{I}_d$. Therefore, we have that with probability at least $1-n^2F(d)$, $\mathbf{\Omega}$ is a matrix with $\Omega_{i, i}=1$ for all $i \in [n]$ and for every $i \in V$ and $j \in V$ such that the directed edge $i \rightarrow j$ or $j \rightarrow i$ doesn't exist, we have $\Omega_{i, j} \in \mathcal{I}_d$. 
 		
		 \xhdr{Proof of part (1).} 
		 
		 Consider any $I \subseteq [n]$ such that $|I| \leq k^2$. From the circle Lemma~\ref{lem:circleTheorem} we have that $\| \myOmegaSub{I}{I} \| \leq 1 + \frac{k^2 C_{conc}}{d^{0.25}} = 1 + \mathcal{J}(k, n)$. From \ref{prop:6} we have, $\| \myOmegaSub{I}{I}^{-1} \| = \frac{1}{\sigma_n( \myOmegaSub{I}{I})}$. From the circle Lemma~\ref{lem:circleTheorem} we have that $\sigma_n( \myOmegaSub{I}{I} )  \geq 1-\mathcal{J}(k, n)$. Thus, $\| \myOmegaSub{I}{I}^{-1} \| \leq \frac{1}{1-\mathcal{J}(k, n)} \leq 1+2 \mathcal{J}(k, n)$.

 		
 		\xhdr{Proof of part (2).}

 		Consider $I, J \subseteq [n]$ such that $|I| \leq k^i$ and $|J| \leq k^j$. From \ref{prop:5} we have, $\| \myOmegaSub{I}{J} \| = \sigma_1(\myOmegaSub{I}{J})$, where $\sigma_1(.)$ is the largest singular value function. Note that \\ $\sigma_1(\myOmegaSub{I}{J}) = \sqrt{\lambda_1(\myOmegaSub{I}{J}^T \myOmegaSub{I}{J})}$ where $\lambda_1(.)$ is the largest eigenvalue function. Thus, we want to upper-bound the quantity $\| \myOmegaSub{I}{J} \| := \sqrt{\lambda_1(\myOmegaSub{I}{J}^T \cdot \myOmegaSub{I}{J})}$. Let $\vec{A} := \myOmegaSub{I}{J}^T \cdot \myOmegaSub{I}{J}$. Since $|I| \leq k^2$, $\vec{A}$ is a $k^2 \times k^2$ matrix. Then for every $i \in [k^2]$ there exists a $h \in [n]$ such that $A_{i, i} = \Omega_{h, h}^2$. Likewise, for every $i, j \in [k^2]$ such that $i \neq j$ there exists $h_1 \neq h_2 \in [n]$ and $h_3 \neq h_4 \in [n]$ such that $A_{i, j} = \Omega_{h_1, h_2} \Omega_{h_3, h_4}$. From the circle Lemma~\ref{lem:circleTheorem}, the largest eigenvalue of a matrix is upper-bounded by the sum of absolute values in any row. Using Lemma~\ref{lem:Omega} and the fact that $|I| \leq k^i$ this is at most $1 + k^{\max\{i, j\}/2} \frac{C_{conc}^2}{d^{0.5}}$. Thus, we get that $\sqrt{\lambda_1(\myOmegaSub{I}{J}^T \cdot \myOmegaSub{I}{J})} \leq k^{\max\{i, j\}/2} \ast \sqrt{\left(1+ \frac{C_{conc}^2}{d^{0.5}} \right)} \leq k^{\max\{i, j\}/2} \ast \left(1+ \sqrt{\frac{C_{conc}^2}{d^{0.5}}}\right) \leq k^{\max\{i, j\}/2} \ast \left( 1+ \frac{\mathcal{J}(k, n)}{k} \right)$.
	\end{proof}
	
	\subsection{Proof of Lemma~\ref{lem:lambdaRand}}
	\begin{proof}
		From \ref{prop:4} we have that $\| \myLambda \| = \sqrt{\lambda_1(\myLambda^T \cdot \myLambda)}$. Note that absolute value of every entry in $\myLambda$ is upper-bounded by $\frac{1}{2 \mu k}$. Now consider any column $v$ in $\myLambda$. The set of non-zero entries in this row correspond to entries $\Lambda_{v, w}$ such that $w \in \pa{v}$. Now consider $\myLambda^T \cdot \myLambda$.
		
		Define $n_{v, w}$ to denote the number of common parents of vertices $v$ and $w$. The total number of non-zero entries in $(\myLambda^T \cdot \myLambda)_{v, w}$ in a given row is at most $\sum_{w} n_{v, w}$. From Lemma~\ref{lem:circleTheorem} we have that, 
			\[
					\lambda_1(\myLambda^T \cdot \myLambda) \leq \sqrt{\sum_{w} n_{v, w} \cdot \frac{1}{4k^2 \mu^2}} \leq \frac{1}{\mu}. \qedhere
			\]
			The last inequality above follows from the following argument. We claim that $\sum_{w} n_{v, w} \leq k^2$. To see this, consider the following counting argument. Construct a bipartite graph $G = (I = \pav, J = \ch{\pav})$ where an undirected edge from $i \in \pav$ to $j \in \ch{\pav}$ exists if and only if there is a directed edge from $i$ to $j$ in the causal model. Thus, we want to compute $\sum_{w} n_{v, w}$ which equals the sum of degree of vertices in $J$. This equals to the sum of degrees of vertices in $I$. Note that $|\pav| \leq k$ and each vertex has at most $k$ out-degree edges in the causal model. Thus, the sum of degrees of vertices in $I$ is equal to the $k^2$.

		Similarly, let $I, J \subseteq [n]$ be two subsets of indices. From \ref{prop:11} and the proof above, this implies that $\| \myLambdaSub{I}{J} \| \leq \frac{1}{\mu}$.

	\end{proof}
	\subsection{Proof of Claim~\ref{lem:dBound}}
	\begin{proof}
		Note that $d \geq \dbound$. Thus $\mathcal{J}(k, n) = \frac{k^2 C_{conc}}{d^{0.25}} \leq \frac{C_{conc}}{\log n} = o(1)$.	
	\end{proof}

\section{Experiments}
In this section, we describe the results of our simulation studies (more can also be found in the supplementary materials). We consider general \emph{bow-free} graphs and random noise. Before we describe the experimental procedure, we briefly describe the challenges in running experiments; this explain why experiments in prior works are almost non-existent. The key issue with experimentation is that the ground-truth model is \emph{unknown} and the datasets do not come with the true underlying model. In particular, $\LSEM$ is a model-based approach where designing the right model is part of the hypothesis held by the experimenter. The dataset only contains the observational data; part of the challenge in inferring causality using $\LSEM$ is in devising an appropriate model based on domain knowledge. Thus, here and in prior works (\cite{RICF, ourUAI19}) the experimental setup \emph{simulates} various possible hypotheses in the hypothesis space.

		\begin{figure*}[!ht]
			\minipage{0.48\textwidth}
			\includegraphics[width=\linewidth]{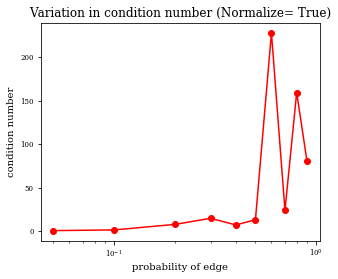}
			\endminipage\hfill
			\minipage{0.48\textwidth}
			\includegraphics[width=\linewidth]{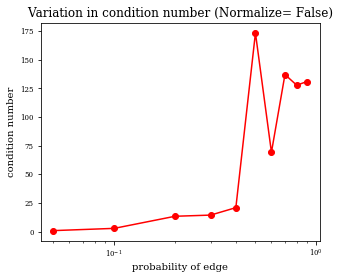}
			\endminipage
			\caption{Experimental results on the gene expression dataset.}
			\label{fig:geneResults}
		\end{figure*}

\xhdr{Gene expression dataset.} We use the dataset that corresponds to experiments on gene expression in \emph{Arabidopsis thaliana} from \cite{wille2004sparse}. We look at the 13 genes which belong to a single pathway: DXPS1, DXPS2(cla1), DXPS3, MCT, DXR, PPDS1, PPDS2mt, GPPS, IPPI1, HDR, HDS, MECPS.
	There are $n=118$ microarray experiments. Thus, the input matrix $\vec{X} \in \mathbb{R}^{118 \times 13}$. We have $13$ vertices, one corresponding to each of the genes. First, we choose a random permutation $\pi$ to order the vertices. For any pair of vertices $i, j$ such that $\pi(i) < \pi(j)$ we add a directed edge from $i$ to $j$ with probability $p$. For every vertex $j$, we choose a vertex $i \neq pa(j)$ uniformly at random and add a bidirected edge between $i$ and $j$. For every other pair of vertices, if there exists no directed edge between them, we add a bidirected edge with probability $0.1$. For a given value of $p$, we generate $30$ random graph structures using the above procedure. To evaluate the condition number, we choose $\epsilon  = 10^{-1}$ and add independent $\mathcal{N}(0, \epsilon^2)$ noise to each entry in the matrix $\vec{X}$ to obtain the perturbed dataset $\tilde{\vec{X}}_{\epsilon}$. We then compute the corresponding covariance matrix $\tilde{\vec{\Sigma}}_{\epsilon}$. We use the algorithm in \cite{FDD2012Annals} to recover parameters $\myLambda$ and $\tilde{\myLambda}_{\epsilon}$ corresponding to the matrices $\mySigma$ and $\tilde{\mySigma}_{\epsilon}$. For a given realization of the random graph, we generate $20$ different datasets $\tilde{\vec{X}}_{\epsilon}$. For each of these $20$ datasets, we compute the corresponding covariance matrices and run the parameter recovery algorithm~\cite{FDD2012Annals} on them. We then average the condition numbers (\ie maximum relative change in $\myLambda$ to the maximum relative change in $\mySigma$) across various values realizations of the random graph. Thus, a single experiment is averaged over the $30$ different random graphs multiplied by the $20$ different runs for a fixed graph. We run two kinds of experiments for each $p$: (1) in which we normalize the dataset (\ie every row in the matrix $\vec{X}$ has a norm of $1$) (2) in which the dataset is not normalized. Figure~\ref{fig:geneResults} shows the results of our experiments. We run simulations for $p \in \{ 0.05, 0.1, 0.2, 0.3, \ldots, 0.9\}$. As can be seen from the results when the values of $p$ are small (sparse regime), the average condition number tends to be small. However, as the value of $p$ becomes large (dense models) the condition number increases. This can be explained by the fact that when errors across many edges accumulate, the total error gets compounded.
	
\xhdr{Assumptions (A.1)-(A.3).} We use the above dataset to inspect the \emph{practicality} of the assumptions in Model~\ref{mod:mainModel}. To do so, we do the following. We consider random DAGs as in the previous subsection and use the matrix $\vec{X}$ as the observational data. As commonly done in practice, we normalize each row $\vec{X}_i$ in the matrix $\vec{X}$ such that $\| \vec{X}_i \| \leq 1$. For each realization of the random DAG, we use the same noise model as above and check if each of the assumptions (A.1)-(A.3) hold. We consider 20 different random DAGs. In all 20 random realizations, we notice that assumptions (A.1), (A.2) and (A.3) hold. This indicates that the assumptions considered in this paper is applicable in practice.

\subsection{Simulated Dataset}
		\label{subsec:simulated}
		We also perform simulations on larger datasets that are synthetically generated. We consider two sets of experiments that differ in the number of vertices in any layer: we consider $k=2$ and $k=7$. For each setting of $k$, we consider $p=0.2$ (sparse regime) and $p=0.8$ (dense regime). When $k=2$ we consider graphs where the total number of vertices is in the set $\{20, 30, 40, 50\}$ while when $k=7$ the number of vertices were in the set $\{ 14, 21, 35,49\}$. For each triple $(k, p, n)$, we generate many random graphs exactly as in the main section of the experiments. We generate a random $\myLambda$ corresponding to the random graph instance, where every edge is given a weight uniformly at random from $[-\range, \range]$. We use two values of $\range$ in the experiments ($\range=1/7$ and $\range=1$). For every bidirected edge between $(i, j)$ we sample a $\mathcal{N}(0, 1)$ random variable $\omega$ and let both $\myOmega_{i, j} = \myOmega_{j, i} = \omega$. For every $i \in [n]$ we let $\Omega_{i, i}$ to be the sum of absolute values in row $i$ added to a $\chi_1^2$-random variable.\footnote{This is the exact setup in \cite{RICF}.} The construction implies that $\myOmega$ is a Symmetric Diagonally Dominant matrix and thus, is Positive Definite. We compute the covariance matrix from Eq.~\eqref{eq:SigmaLSEM}. To compute the condition number, we consider $50$ samples from this model and construct the sample covariance matrix $\tilde{\mySigma}$ which constitutes our perturbed instance. We then compute the average condition number between the exact computation of $\mySigma$ and the one obtained via finite samples. Figure~\ref{fig:simulatedResults}, \ref{fig:simulatedResultsE1}, \ref{fig:simulatedResultsE2} and \ref{fig:simulatedResultsE3} denotes the results of this experiment. As can be seen, in the sparse regime the condition number is fairly low, while in the dense regime the condition number is almost a factor of $10^2$. Thus, these results indicate two things. First, it verifies the claim in this paper that when the assumptions on $\range$ are satisfied the instances are well-conditioned. Second, it also seems to indicate that when $\range$ is large, then some of the assumptions in Model~\ref{mod:mainModel} are also necessary. 

	\begin{figure*}[!ht]
			\minipage{0.48\textwidth}
			\includegraphics[width=\linewidth]{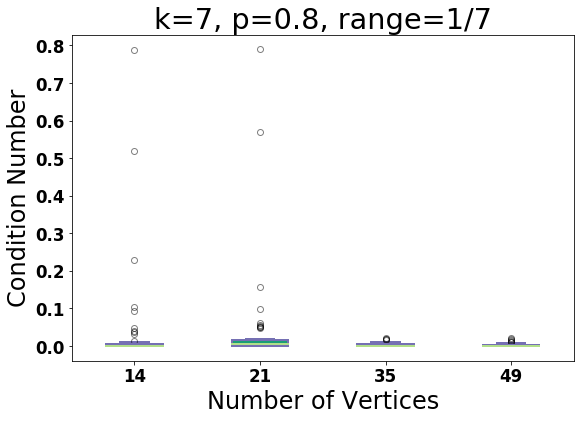}
			\endminipage
			\minipage{0.48\textwidth}
			\includegraphics[width=\linewidth]{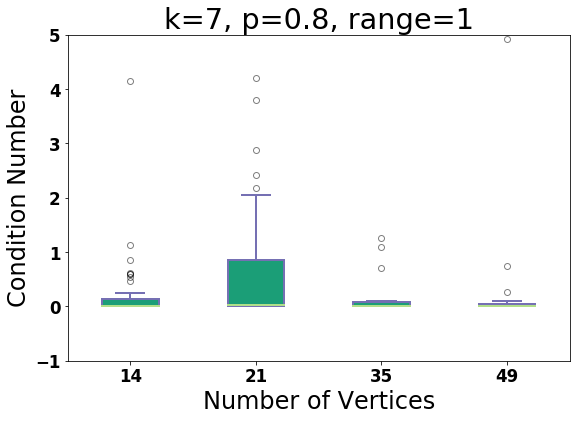}
			\endminipage
			\caption{$k=7$, $p=0.8$}
			\label{fig:simulatedResults}
		\end{figure*}
		
		\begin{figure*}[!ht]
			\minipage{0.48\textwidth}
			\includegraphics[width=\linewidth]{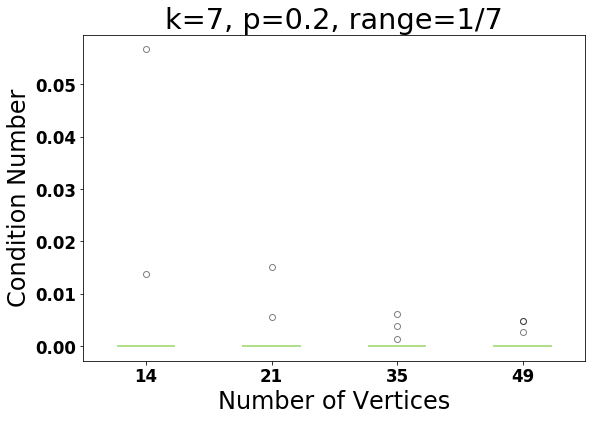}
			\endminipage\hfill
			\minipage{0.48\textwidth}
			\includegraphics[width=\linewidth]{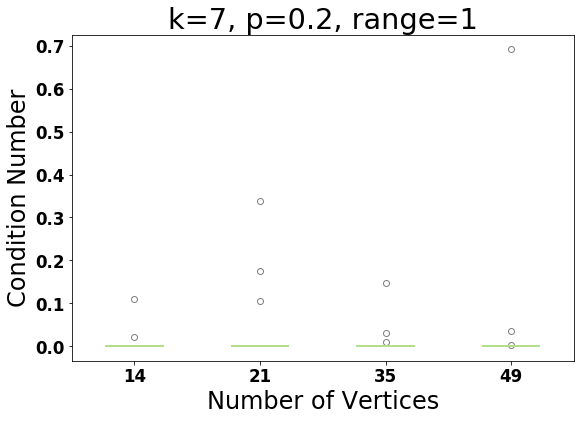}
			\endminipage\hfill
			\caption{$k=7$, $p=0.2$.}
			\label{fig:simulatedResultsE3}
		\end{figure*}

			\begin{figure*}[!ht]
			\minipage{0.48\textwidth}
			\includegraphics[width=\linewidth]{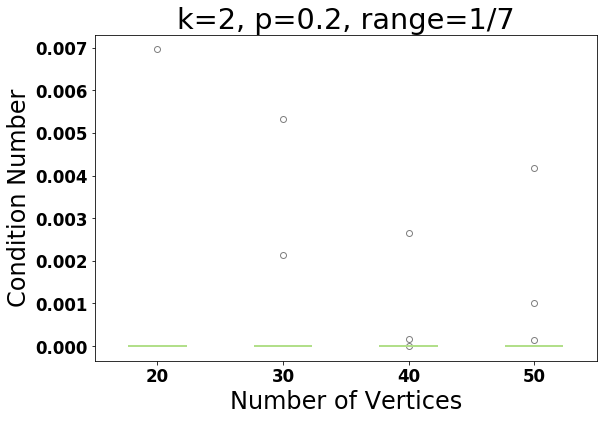}
			\endminipage\hfill
			\minipage{0.48\textwidth}
			\includegraphics[width=\linewidth]{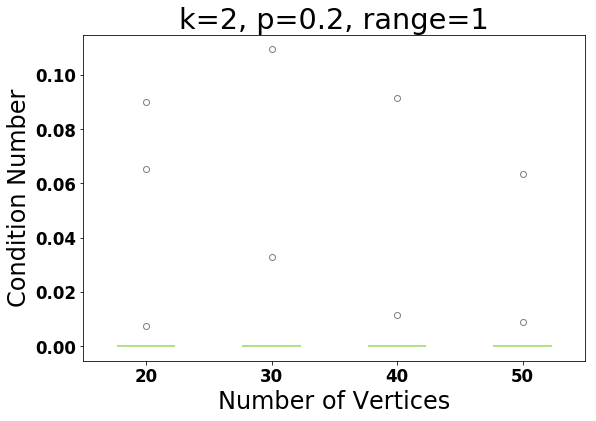}
			\endminipage
			\caption{$k=2$ and $p=0.2$.}
			\label{fig:simulatedResultsE1}
		\end{figure*}
		
		\begin{figure*}[!ht]
			\minipage{0.48\textwidth}
			\includegraphics[width=\linewidth]{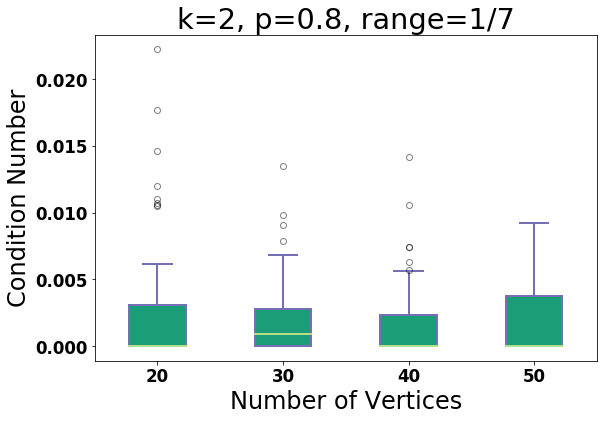}
			\endminipage\hfill
			\minipage{0.48\textwidth}
			\includegraphics[width=\linewidth]{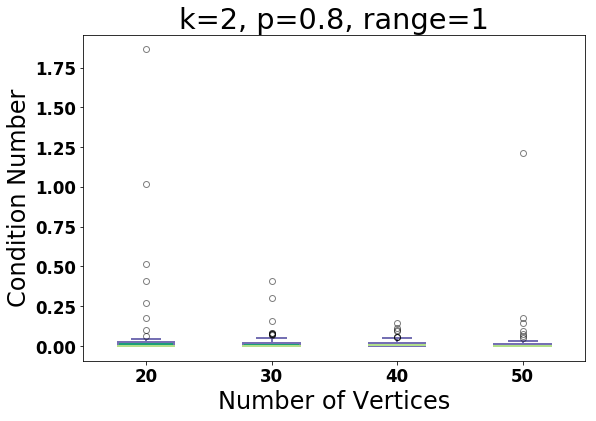}
			\endminipage
			\caption{$k=2$ and $p=0.8$.}
			\label{fig:simulatedResultsE2}
		\end{figure*}

\section{Conclusion}
	\label{sec:Conclusion}
	
	In this paper, we consider the problem of \emph{robust identifiability} in bow-free $\LSEM$s. We give a sufficient condition when bow-free $\LSEM$s can be identified in a robust manner. Our work suggests several directions for future work. First, it would be nice if the sufficient condition (particularly, \ref{mod:offDiagonal} is the most restrictive assumption) can be relaxed. The second direction is to provide sufficient conditions for robust identifiability in other models of causal inference, particularly the semi-Markovian model (note that Proposition 1.3 in \cite{SS2016UAI} is one such sufficient condition). Finally, another direction is to combine robust identifiability with \emph{model misspecification} (\eg \cite{Pearl_ICML19}) where all edges in the model are not correctly specified. Existing works assume access to the \emph{exact} covariance matrix. Thus, combining noisy data with model misspecification would be very interesting.

\section{Acknowledgements}
	AL was supported in part by SERB Award ECR/2017/003296 and a Pratiksha Trust Young Investigator Award. AL is also grateful to Microsoft Research for supporting this collaboration.

\bibliographystyle{acm}
\bibliography{../references}	

\begin{thebibliography}{10}

\bibitem{arora2009expander}
{\sc Arora, S., Rao, S., and Vazirani, U.}
\newblock Expander flows, geometric embeddings and graph partitioning.
\newblock {\em Journal of the ACM (JACM) 56}, 2 (2009), 5.

\bibitem{25}
{\sc Bentler, P.~M., and Weeks, D.~G.}
\newblock Linear structural equations with latent variables.
\newblock {\em Psychometrika 45}, 3 (1980), 289--308.

\bibitem{Bollen}
{\sc Bollen, K.~A.}
\newblock {\em Structural Equations with Latent Variables}.
\newblock Wiley-Interscience, 1989.

\bibitem{BP2006UAI}
{\sc Brito, C., and Pearl, J.}
\newblock {Graphical Condition for Identification in recursive SEM}.
\newblock {\em Proceedings of the Twenty-Second Conference on Uncertainty in
  Artificial Intelligence (UAI)\/} (2006), 47--54.

\bibitem{Pearl_ICML19}
{\sc Cinelli, C., Kumor, D., Chen, B., Pearl, J., and Bareinboim, E.}
\newblock Sensitivity analysis of linear structural causal models.
\newblock In {\em ICML\/} (2019).

\bibitem{Diakonikolas-Kane}
{\sc Diakonikolas, I., and Kane, D.~M.}
\newblock Recent advances in algorithmic high-dimensional robust statistics.
\newblock {\em CoRR abs/1911.05911\/} (2019).

\bibitem{RICF}
{\sc Drton, M., Eichler, M., and Richardson, T.~S.}
\newblock Computing maximum likelihood estimates in recursive linear models
  with correlated errors.
\newblock {\em Journal of Machine Learning Research 10}, Oct (2009),
  2329--2348.

\bibitem{51}
{\sc Drton, M., Foygel, R., and Sullivant, S.}
\newblock Global identifiability of linear structural equation models.
\newblock {\em The Annals of Statistics\/} (2011), 865--886.

\bibitem{DW2016Scandinavian}
{\sc Drton, M., and Weihs, L.}
\newblock {Generic Identifiability of Linear Structural Equation Models by
  Ancestor Decomposition}.
\newblock {\em Scandinavian Journal of Statistics 43}, 4 (2016), 1035--1045.

\bibitem{FDD2012Annals}
{\sc Foygel, R., Draisma, J., and Drton, M.}
\newblock {Half-trek criterion for generic identifiability of linear structural
  equation models}.
\newblock {\em Annals of Statistics 40}, 3 (2012), 1682--1713.

\bibitem{Ghoshal_NIPS}
{\sc Ghoshal, A., and Honorio, J.}
\newblock Learning identifiable gaussian bayesian networks in polynomial time
  and sample complexity.
\newblock In {\em Advances in Neural Information Processing Systems 30},
  I.~Guyon, U.~V. Luxburg, S.~Bengio, H.~Wallach, R.~Fergus, S.~Vishwanathan,
  and R.~Garnett, Eds. Curran Associates, Inc., 2017, pp.~6457--6466.

\bibitem{Ghoshal_AISTATS}
{\sc Ghoshal, A., and Honorio, J.}
\newblock Learning linear structural equation models in polynomial time and
  sample complexity.
\newblock In {\em Proceedings of the Twenty-First International Conference on
  Artificial Intelligence and Statistics\/} (Playa Blanca, Lanzarote, Canary
  Islands, 09--11 Apr 2018), A.~Storkey and F.~Perez-Cruz, Eds., vol.~84 of
  {\em Proceedings of Machine Learning Research}, PMLR, pp.~1466--1475.

\bibitem{64}
{\sc Guyon, I., Janzing, D., and Scholkopf, B.}
\newblock Causality: Objectives and assessment.
\newblock In {\em Causality: Objectives and Assessment\/} (2010), pp.~1--42.

\bibitem{67}
{\sc Holland, P.~W., Glymour, C., and Granger, C.}
\newblock Statistics and causal inference.
\newblock {\em ETS Research Report Series 1985}, 2 (1985).

\bibitem{Maclaren}
{\sc Maclaren, O., and Nicholson, R.}
\newblock What can be estimated? identifiability, estimability, causal
  inference and ill-posed inverse problems.
\newblock {\em arXiv:1904.02826\/} (2019).

\bibitem{85}
{\sc McDonald, R.~P.}
\newblock What can we learn from the path equations?: Identifiability,
  constraints, equivalence.
\newblock {\em Psychometrika 67}, 2 (2002), 225--249.

\bibitem{mitzenmacher2005probability}
{\sc Mitzenmacher, M., and Upfal, E.}
\newblock {\em Probability and computing: Randomized algorithms and
  probabilistic analysis}.
\newblock Cambridge university press, 2005.

\bibitem{pearlBook}
{\sc Pearl, J.}
\newblock {Causality}.
\newblock {\em Cambridge university press\/} (2009).

\bibitem{PearlMackenzie18}
{\sc Pearl, J., and Mackenzie, D.}
\newblock {\em The Book of Why}.
\newblock Basic Books, New York, 2018.

\bibitem{95}
{\sc Peters, J., Janzing, D., and Sch\"{o}lkopf, B.}
\newblock {\em Elements of Causal Inference - Foundations and Learning
  Algorithms}.
\newblock Adaptive Computation and Machine Learning Series. The MIT Press,
  Cambridge, MA, USA, 2017.

\bibitem{matrixBook}
{\sc Petersen, K.~B., and Pedersen, M.~S.}
\newblock The matrix cookbook.
\newblock {\em Technical University of Denmark 7.15\/} (2008).

\bibitem{ourUAI19}
{\sc Sankararaman, K.~A., Louis, A., and Goyal, N.}
\newblock Stability of linear structural equation models of causal inference.
\newblock In {\em Proceedings of the 35th Conference on Uncertainty in
  Artificial Intelligence\/} (2019), UAI '19.

\bibitem{SS2016UAI}
{\sc Schulman, L.~J., and Srivastava, P.}
\newblock {Stability of Causal Inference}.
\newblock {\em Uncertainty in Artificial Intelligence (UAI)\/} (2016).

\bibitem{Shpitser2008}
{\sc Shpitser, I., and Pearl, J.}
\newblock {Complete Identification Methods for the Causal Hierarchy}.
\newblock {\em Journal of Machine Learning Research 9\/} (2008), 1941--1979.

\bibitem{skiena1991implementing}
{\sc Skiena, S.}
\newblock Implementing discrete mathematics: combinatorics and graph theory
  with mathematica.

\bibitem{srivastava2013covariance}
{\sc Srivastava, N., Vershynin, R., et~al.}
\newblock Covariance estimation for distributions with 2 + e moments.
\newblock {\em The Annals of Probability 41}, 5 (2013), 3081--3111.

\bibitem{stewart1998matrix}
{\sc Stewart, G.~W.}
\newblock {\em Matrix Algorithms: Volume 1: Basic Decompositions}, vol.~1.
\newblock Siam, 1998.

\bibitem{strang1993introduction}
{\sc Strang, G.}
\newblock {\em Introduction to linear algebra}, vol.~3.
\newblock Wellesley-Cambridge Press Wellesley, MA, 1993.

\bibitem{108}
{\sc Terry, L.}
\newblock Smoking and health.
\newblock {\em The Reports of the Surgeon General\/} (1964).

\bibitem{wille2004sparse}
{\sc Wille, A., Zimmermann, P., Vranov{\'a}, E., F{\"u}rholz, A., Laule, O.,
  Bleuler, S., Hennig, L., Preli{\'c}, A., von Rohr, P., Thiele, L., et~al.}
\newblock Sparse graphical gaussian modeling of the isoprenoid gene network in
  arabidopsis thaliana.
\newblock {\em Genome biology 5}, 11 (2004), R92.

\end{thebibliography}

\newpage
\onecolumn

\section{Missing Proofs in Section~\ref{sec:randomModel}}

	\subsection{Proof of Lemma~\ref{lem:Omega}}
		\begin{proof}
			The proof of this is similar to \cite{ourUAI19}. We include the full proof here for completeness. 
			
			Generate $n$ vectors $\mathbf{\tilde{v}}_1$, $\mathbf{\tilde{v}}_2, \ldots, \mathbf{\tilde{v}}_n \in \mathbb{R}^d$ independently from the $d$-dimensional unit sphere. Let $\mathbf{v}_u = \mathbf{\tilde{v}}_u$ for all $u$ such that $\pa{u} = \phi$. Consider a vertex $u \in V$ such that $\pa{u} \neq \phi$. To generate $\mathbf{v}_u$, we first remove all components of $\mathbf{\tilde{v}}_u$ parallel to $\mathbf{v}_{u'-1}$ for vertices $u' \in \pa{u}$ and then normalize the resultant vector. Define 
				\[
						\hat{\vec{v}}_u := \mathbf{\tilde{v}}_u -  \smashoperator{\sum_{u' \in \pa{u}}} \langle \mathbf{\tilde{v}}_u, \mathbf{v}_{u'} \rangle \mathbf{v}_{u'}.
				\] 
				Then $\mathbf{v}_u$ can be formally written as,
		 		\[
		 				\mathbf{v}_{u} = \frac{\hat{\vec{v}}_u}{ \| \hat{\vec{v}}_u \| }.
		 		\]
		We will now prove the statement in the lemma. Consider a pair $i \in V$ and  $j \in V$ such that $i$ appears before $j$ in the topological sort order and $i \not \in \pa{j}$. We will show that for a given pair with probability at least $1-F(d)$, the statement in the lemma holds.
		
		Consider $\Abs{\langle \vec{v}_i, \vec{v}_j \rangle}$. Using Lemma~\ref{lem:gaussianprojections}, we have that with probability at least $1-\frac{F(d)}{k+1}$ each of the following holds.
		\begin{enumerate}
			\item $\Abs{\langle \tilde{\vec{v}}_i, \vec{v}_j \rangle} \in \left[ -\frac{C_{\text{conc}}}{(2k+1)d^{0.25}}, \frac{C_{\text{conc}}}{(2k+1)d^{0.25}} \right]$
			\item $ \Abs{\langle \tilde{\vec{v}}_i, \vec{v}_{u'} \rangle} \in  \left[ -\frac{C_{\text{conc}}}{(2k+1)d^{0.25}}, \frac{C_{\text{conc}}}{(2k+1)d^{0.25}} \right]$ for every $u' \in \pa{i}$.
		\end{enumerate}
		
		Thus, taking a union bound, with probability at least $1-F(d)$ all of them hold simultaneously. In what follows, we will condition on these events.
		\begin{align}
			\Abs{\langle \vec{v}_i, \vec{v}_j \rangle} &  = \Abs{\langle \tilde{\vec{v}}_i, \vec{v}_j \rangle} + \Abs{\langle \vec{v}_i - \tilde{\vec{v}_i}, \vec{v}_j \rangle}. & \nonumber \\
			&\leq \frac{C_{\text{conc}}}{(2k+1)d^{0.25}} + \Norm{\vec{v}_i-\tilde{\vec{v}}_i}. \label{eq:align1} \\
			& \leq \frac{C_{\text{conc}}}{(2k+1)d^{0.25}} + \Norm{\vec{v}_i - \hat{\vec{v}}_i} + \Norm{\hat{\vec{v}}_i - \tilde{\vec{v}}_i}. \label{eq:align2} \\
			& \leq \frac{C_{\text{conc}}}{(2k+1)d^{0.25}} + \Norm{\vec{v}_i - \hat{\vec{v}}_i} + \smashoperator{\sum_{u' \in \pa{i}}} \Abs{\langle \tilde{\vec{v}}_i, \vec{v}_{u'} \rangle}  \label{eq:align3} \\
			& \leq \frac{(k+1)C_{\text{conc}}}{(2k+1)d^{0.25}}  + \Norm{\vec{v}_i - \hat{\vec{v}}_i}  \label{eq:align4}
		\end{align}
		In Eq.~\eqref{eq:align1}, the first summand follows from the high-probability event (1) above. Eq.~\eqref{eq:align2} follows from triangle inequality. Eq.~\eqref{eq:align3} follows from the definition of $\hat{\vec{v}}_i$. Eq.~\eqref{eq:align4} follows from high-probability event (2) above.
		
		We will now show that $\Norm{\vec{v}_i - \hat{\vec{v}}_i} \leq \frac{k C_{\text{conc}}}{(2k+1)d^{0.25}}$. This will complete the proof. Note from the definition of $\hat{\vec{v}}_i$ we have that $\Norm{\vec{v}_i - \hat{\vec{v}}_i} = \Norm{\frac{\hat{\vec{v}}_i}{\Norm{\hat{\vec{v}}_i}} -\hat{\vec{v}}_i} = \Norm{\hat{\vec{v}}_i}\left( 1- \frac{1}{\Norm{\hat{\vec{v}}_i}} \right)$. Moreover we have that $\Norm{\hat{\vec{v}}_i} \leq \Norm{\tilde{\vec{v}}_i} + \smashoperator{\sum_{u' \in \pa{i}}}\Abs{\langle \tilde{\vec{v}}_i, \vec{v}_{u'} \rangle}$. From (2) above we have that each of the second term lies in $\left[ -\frac{C_{\text{conc}}}{(2k+1)d^{0.25}}, \frac{C_{\text{conc}}}{(2k+1)d^{0.25}} \right]$. The first term is $1$ since it is a unit vector. Thus $\Norm{\hat{\vec{v}}_i}\left( 1- \frac{1}{\Norm{\hat{\vec{v}}_i}} \right) \leq \frac{k C_{\text{conc}}}{(2k+1)d^{0.25}}$.
	\end{proof}

\section{Technical Lemmas}
	\begin{lemma}[Properties of Spectral Norm (see Chapter 7 in \cite{strang1993introduction})]
		\label{lem:normProp}
		The spectral norm $\| . \|$ satisfies the following properties for any two square matrices $\vec{A}$ and $\vec{B}$.
		\begin{enumerate}[label=\textbf{(P.\arabic*)},ref=Prop. (P.\arabic*)]
			\item \label{prop:1} $\| \vec{A} + \vec{B} \| \leq \| \vec{A} \| + \| \vec{B} \|$.
			\item \label{prop:1a} $\| \vec{A} - \vec{B} \| \geq | \| \vec{A} \| - \| \vec{B} \| |$.
			\item \label{prop:2} $\| \vec{A} \cdot \vec{B} \| \leq \| \vec{A} \| \| \vec{B} \|$.
			\item \label{prop:3} $\| \vec{A}^T \| = \| \vec{A} \|$.
			\item \label{prop:4} $\| \vec{A} \| = \sigma_1(\vec{A})$ where $\sigma_1(.)$ represents the largest singular value function. If $\vec{A}$ is symmetric $\sigma_1( \vec{A} )$ coincides with the largest eigen value of $\vec{A}$. If $\vec{A}$ is not symmetric we have the relationship $\sigma_1(\vec{A}) = \sqrt{\lambda_1(\vec{A}^T \vec{A})}$.
			\item \label{prop:5} For any entry $(i, j)$, we have $| A_{i, j} | \leq \| \vec{A} \|$.
			\item \label{prop:6} Let $\vec{A}$ be an invertible matrix. Then $\| \vec{A}^{-1} \| = \frac{1}{\sigma_n(\vec{A})}$, where $\sigma_n(.)$ is the smallest singular value function.
			\item \label{prop:7} The largest eigen value $\lambda_1(\vec{A}) \leq \Tr[\vec{A}]$.
			\item \label{prop:8} If $\vec{A} \in \mathbb{R}^{n \times n}$ then $\sqrt{\Tr[\vec{A}^T \cdot \vec{A}]} \leq \sqrt{n} \| \vec{A} \|$.
			\item \label{prop:9} Let $\vec{A}$ be a matrix. Then the spectral norm of the extended matrices satisfy the following. $\| \left[ \vec{A} \quad \vec{0} \right] \| = \| \left[ \vec{0} \quad \vec{A} \right] \| = \|\vec{A} \|$.
			\item \label{prop:11} Let $\vec{A} \in \mathbb{R}^{n \times k}$. Then we have $\|\vec{A} \| \leq \max_{j \in k} k \times \| \vec{A}_j \|$, where $\vec{A}_j$ represents the $j^{th}$ column. Moreover for any sub-matrix $\vec{A'}$ of $\vec{A}$, we have that $\| \vec{A'} \| \leq \| \vec{A} \|$.
		\end{enumerate}
		Additionally, let $\vec{A} = \vec{B} - \vec{C}$ where $\vec{C} \succeq 0$ is a PSD matrix. Then we have $\| \vec{A} \| \leq \| \vec{B} \|$.
	\end{lemma}
	
	\begin{lemma}[Gershgorin circle lemma]
		\label{lem:circleTheorem}
		Let $\vec{A} \in \mathbb{R}^{n \times n}$ be a real symmetric matrix. Let $R_i := \sum_{j \neq i} |A_{i,j}|$. Then every eigen value $\lambda$ of the matrix $\vec{A}$ lies in one of the union of intervals $\cup_{i \in [n]} [A_{i, i} - R_i, A_{i, i} + R_i]$ . 
	\end{lemma}
	
	\begin{lemma}[Uniform distribution]
	\label{appx:uniform}
	Let $X$ be a random variable that is uniformly distributed in the interval $[-h, h]$. Then we have the following.
	\begin{enumerate}
		\item The mean $\mathbb{E}[X] = 0$ and the variance $\operatorname{Var}(X)=\mathbb{E}[X^2] = h^2/3$.
		\item For any given $0 \leq \beta \leq 1$ we have that $\Pr{-\beta \leq X \leq \beta} = \beta/h$.
	\end{enumerate}
 \end{lemma}	
	The proof of these theorems follow directly from the definition of a uniform distribution and we refer the reader to a standard textbook on probability (\cite{mitzenmacher2005probability}).

	We state a Lemma on the gaussian behavior of unit vectors. This can be found in Lemma 5 of \cite{arora2009expander}.
	\begin{lemma}[Gaussian behavior of projections]
		\label{lem:gaussianprojections}
		Let $\mathbf{v}$ be an arbitrary unit vector in $\mathbb{R}^d$. Let $\mathbf{u}$ be a randomly chosen unit vector of dimension $d$. Then for every $x \leq \sqrt{d}/4$ we have,
		\[
				\Pr{\Abs{\langle \mathbf{v}, \mathbf{u} \rangle} \geq \frac{x}{\sqrt{d}}	} \leq \exp(-x^2/4).
		\]
		As a corollary, for an absolute constant $C_{\text{conc}} > 0$ and $x = C_{\text{conc}} \cdot d^{0.25}$ we have that,
		\[
				\Pr{\Abs{\langle \mathbf{v}, \mathbf{u} \rangle} \geq \frac{C_{\text{conc}}}{d^{0.25}} } \leq \exp\left[ -\Omega \left( \sqrt{d} \right) \right].
		\]
	\end{lemma}
	
	The following lemma on matrix approximation is used in our proofs. This can be found in \cite{matrixBook}.
	\begin{lemma}[Matrix Approximation]
		\label{lem:matrixApprox}
		Let $\vec{Q}$ and $\vec{M}$ be two matrices. If $ \| \vec{Q}^{-1} \vec{M} \| \leq \tfrac{1}{2}$, then
		\[
				\| (\vec{Q} + \vec{M})^{-1} \| \leq \frac{\| \vec{Q}^{-1} \|}{1- \| \vec{Q}^{-1} \vec{M} \|} \leq \| \vec{Q}^{-1} \| (1+2 \| \vec{Q}^{-1} \vec{M} \|).
		\]
	\end{lemma}
	\begin{proof}
		The first inequality follows from Corollary 4.19 in \cite{stewart1998matrix}. To obtain the second inequality note that from the Taylor series expansion,
		\[
				\frac{1}{1-x} = 1 + x + x^2 + \ldots,
		\]	
		Moreover, when $x < \frac{1}{2}$, we have $x^2 + x^3 + \ldots = \frac{x^2}{1- x} < x$. Thus, $\frac{1}{1-x} < 1+ 2x$.
	\end{proof}
	\begin{lemma}[Norm of perturbations]
		\label{lem:matrixPerturbations}
		Let $\vec{A}$ be a given invertible matrix and let $\vec{B}$ be any other matrix, such that $\| \vec{A}^{-1} \vec{B} \| < \frac{1}{2}$. Then we have,
		\[
				\| \vec{A}^{-1} - (\vec{A} + \vec{B})^{-1} \| \leq \| \vec{A}^{-1} \| \|\vec{A}^{-1} \vec{B} \| (1 + 2\|\vec{A}^{-1} \vec{B} \|).
		\]
	\end{lemma}
	\begin{proof}
		Consider $\vec{A}^{-1} - (\vec{A} + \vec{B})^{-1}$. This can be written as,
		\begin{align}
			\vec{A}^{-1} - (\vec{A} + \vec{B})^{-1} & = 	 \vec{A}^{-1}  - (\vec{A} (\Identity + \vec{A}^{-1} \vec{B}))^{-1} . \nonumber \\
			& =  \vec{A}^{-1}  - (\Identity + \vec{A}^{-1} \vec{B})^{-1} \vec{A}^{-1} . \nonumber 
		\end{align}
		Using the Taylor series expansion $(\Identity + \vec{P})^{-1} = \Identity - \vec{P} + \vec{P}^2 - \ldots$ (\cite{matrixBook}) we get the above equals,
		\begin{align*}
			& = \vec{A}^{-1} - (\Identity - \vec{A}^{-1} \vec{B} + (\vec{A}^{-1} \vec{B})^2 - \ldots ) \vec{A}^{-1}. \nonumber \\
			& = \vec{A}^{-1} - \vec{A}^{-1} + \left( \vec{A}^{-1} \vec{B} - (\vec{A}^{-1} \vec{B})^2 + \ldots \right) \vec{A}^{-1}. \nonumber \\
			& =  (\vec{A}^{-1} \vec{B} - (\vec{A}^{-1} \vec{B})^2 + \ldots ) \vec{A}^{-1}. \nonumber \\
		\end{align*}
		Now taking the spectral norm on both sides, we get,
		\begin{align*}
			\| \vec{A}^{-1} - (\vec{A} + \vec{B})^{-1}  \| & = \| (\vec{A}^{-1} \vec{B} - (\vec{A}^{-1} \vec{B})^2 + \ldots ) \vec{A}^{-1} \|.
		\end{align*}
		Using the triangle inequality of norm \ref{prop:1}, \ref{prop:2} we get,
		\begin{align*}
			& \leq \| \vec{A}^{-1} \| ( \|\vec{A}^{-1} \vec{B} \| + \|\vec{A}^{-1} \vec{B} \|^2 \ldots ) 
		\end{align*}
		We have that $\|\vec{A}^{-1} \vec{B} \| + \|\vec{A}^{-1} \vec{B} \|^2 \ldots = \frac{\|\vec{A}^{-1} \vec{B} \|}{1-\|\vec{A}^{-1} \vec{B} \|}$. As in the proof of Lemma~\ref{lem:matrixApprox}, using the fact that $\| \vec{A}^{-1} \vec{B} \| < \frac{1}{2}$, we have that,
		\[
				\| \vec{A}^{-1} \| \left( \frac{\|\vec{A}^{-1} \vec{B} \|}{1-\|\vec{A}^{-1} \vec{B} \|} \right) \leq \| \vec{A}^{-1} \| \|\vec{A}^{-1} \vec{B} \| (1 + 2\|\vec{A}^{-1} \vec{B} \|). \qedhere
		\]
	\end{proof}
	\begin{lemma}[Norm of Differences]
		\label{lem:normDifference}
		 Let $\vec{A}, \vec{B}, \vec{C}$ be matrices such that $\vec{A}$ is invertible. Let $\| \vec{A}^{-1} \cdot \vec{B} \cdot \vec{C} \| < 1$. Then we have,	
		\[
				\| (\vec{A} - \vec{B} \cdot \vec{C} )^{-1} \| \leq \| \vec{A}^{-1} \| \left( \frac{1}{1-  \|\vec{A}^{-1} \vec{B} \vec{C} \| } \right).             
		\]
	\end{lemma}
	\begin{proof}
		This follow from Lemma~\ref{lem:matrixApprox} by putting $\vec{Q} = \vec{A}$ and $\vec{M} = - \vec{B} \cdot \vec{C}$.
	\end{proof}

\end{document}